\documentclass[letterpaper]{article} 
\usepackage[preprint]{aaai2027}  
\usepackage[hyphens]{url}  
\usepackage{graphicx} 
\usepackage{natbib}  
\usepackage{caption} 
\usepackage{algorithm}
\usepackage{algorithmic}

\usepackage{booktabs}

\usepackage{amsmath, amssymb, amsfonts, amsthm}
\usepackage{mathtools}
\usepackage{enumitem}
\usepackage{multirow}
\usepackage{pifont} 

\newcommand{\fidbetter}[1]{\,{\tiny\textcolor{red}{($-#1\%$)}}}
\newcommand{\fidworse}[1]{\,{\tiny\textcolor{blue}{($+#1\%$)}}}

\newtheorem{theorem}{Theorem}
\newtheorem{proposition}[theorem]{Proposition}

\newtheorem{corollary}[theorem]{Corollary}
\theoremstyle{remark}

\DeclareMathOperator{\Var}{Var}\DeclareMathOperator{\Cov}{Cov}

\DeclareMathOperator{\tr}{tr}
\DeclareMathOperator{\Lip}{Lip}\DeclareMathOperator{\Bures}{Bures}
\DeclareMathOperator{\Law}{Law}\DeclareMathOperator{\Unif}{Unif}
\DeclareMathOperator*{\argmin}{arg\,min}
\newcommand{\R}{\mathbb{R}}\newcommand{\E}{\mathbb{E}}
\newcommand{\gam}[1]{\gamma_{#1}}\newcommand{\Pperp}{P_U^{\perp}}
\newcommand{\KS}{d_{\mathrm{KS}}}\newcommand{\BL}{d_{\mathrm{BL}}}
\newcommand{\Mix}{\mathrm{Mix}}

\definecolor{cutclr}{rgb}{0.75,0.10,0.10}
\newcommand{\cutrule}{\par\nobreak\noindent\textcolor{cutclr}{%
  \leaders\hbox to 8pt{\hss\rule[0.35ex]{5pt}{0.6pt}\hss}\hfill\kern0pt}\par\nobreak}

\title{One-Sided Quantile Coupling for Flow Matching}

\author{
    Jin-Young Kim\textsuperscript{\rm 1}, So-Yoon Cho\textsuperscript{\rm 2}, Hyun-Gyoon Kim\textsuperscript{\rm 3}\thanks{Corresponding author}
}
\affiliations{
    \textsuperscript{\rm 1}Independent Researcher\\
    \textsuperscript{\rm 2}Division of Finance \& AI, Hankuk University of Foreign Studies \\
    \textsuperscript{\rm 3}Dept. of Financial Engineering, Ajou University  \\
    seago0828@gamil.com, soyooncho@hufs.ac.kr, hyungyoonkim@ajou.ac.kr
}

\begin{document}

\maketitle

\begin{abstract}
{Flow Matching (FM) trains continuous-time generative models by regressing the velocity field of a probability path between a simple source distribution and a target data distribution. The coupling that pairs source and target samples strongly affects optimization and sample quality, but structured couplings typically rely on mini-batch transport or assignment procedures whose cost grows at least quadratically in batch size. We propose \textbf{Quantile Coupling Flow Matching (QC-FM)}, a lightweight \emph{one-sided} coupling: rather than matching two pre-sampled batches, it samples only the data batch and constructs each paired source directly. Data ranks projected along a small number of random orthogonal directions are mapped to Gaussian quantiles, and the latent code is completed in the orthogonal complement by conditional Gaussian sampling. The construction is one-dimensional per slice, so the coupling requires no pairwise cost matrix and no assignment to solve. We show that, for each drawn frame, this coupling eliminates the irreducible regression variance along every selected slice and makes the ideal flow exactly straight there, while leaving the sampling prior unchanged: generation still starts from the standard Gaussian, and the training source deviates from it only through the copula of the slice codes, whose transport cost we bound. For training, we apply QC to an anchor subset and complete the remaining source slots with exact Gaussian samples, retaining the QC bias while preserving an explicit signal from the Baseline coupling. Across CIFAR-10, CelebA, FFHQ, and ImageNet-64, QC-FM improves over the Baseline under matched training budgets, reducing FID by up to $12.9\%$, and outperforms minibatch OT-CFM on all four datasets. These results suggest that preserving projected rank structure is a simple and scalable way to inject useful geometric bias into FM couplings without solving a mini-batch transport problem.}

\end{abstract}


\section{Introduction}

Flow Matching (FM) has emerged as an effective framework for training continuous-time generative models by directly regressing the velocity field of a probability path between a source distribution and a target data distribution \citep{Lip23}. A standard choice uses a Gaussian source and linear interpolation between a source sample and a data sample, yielding a simple and efficient training objective. Despite this simplicity, the quality of the coupling between source and target samples plays a central role in determining the geometry of the induced training trajectories. Poor couplings lead to unnecessarily long or conflicting paths, which can increase the variance of the target velocity field and make training less efficient.
While recent curricula progressively organize the noise-to-data learning
problem across denoising difficulty or temporal gaps
\citep{Kim25Curriculum,Kim26MeanFlow}, we focus on a complementary axis:
directly structuring the noise--data endpoint coupling used at every update.

Structured couplings inspired by optimal transport (OT) can reduce trajectory
crossing and shorten transport paths, often leading to faster convergence and
better generation quality. However, existing structured approaches typically
rely on mini-batch OT or related matching procedures, which
require constructing pairwise cost matrices and solving a batchwise assignment
problem. As data dimension and batch size grow, these procedures become
increasingly expensive, and their batchwise nature can introduce additional
instability or bias.

In this paper, we ask whether the geometric benefits of structured coupling can be obtained without matching two pre-sampled batches at all. Our answer is \textbf{Quantile Coupling Flow Matching (QC-FM)}, a \emph{one-sided} coupling scheme: rather than drawing a noise batch and re-pairing it with the data, QC-FM samples only the data batch and constructs each paired source directly. Every data point is ranked along a small number of random orthogonal directions, and the ranks are mapped to Gaussian quantiles, so that the source preserves the batchwise order of the data along each slice. The remaining directions are completed by conditional Gaussian sampling, and the resulting source is structured along the selected slices and stochastic elsewhere.

QC-FM is a batch-local structured coupling: it imposes monotone one-dimensional transport along $k$ random slices. Each slice is realized by a sort against a fixed Gaussian quantile grid, so there is no pairwise cost matrix to build and no assignment problem to solve. Because the quantiles are computed within the current mini-batch, QC-FM is a surrogate for structured coupling rather than an approximation of global OT; the aim is to recover the optimization benefits commonly associated with structured couplings, not to compute an optimal matching.

The QC coupling is the theoretical primitive of our method. In training, QC designates the source for an \emph{anchor} subset of the batch, and the remaining source slots are completed with exact samples from the Gaussian prior. These {Gaussian-remainder completion} schemes retain a controlled amount of QC structure while preserving an explicit signal from the Baseline coupling. \textbf{QC-FM-Mixture} pairs the Gaussian remainder randomly, and \textbf{QC-FM-Adjacency} uses the same Gaussian remainder but pairs it through neighborhoods induced by the QC anchors, using every data target exactly once. The hybrids are practical completion schemes, rather than separately claimed optimal couplings.

Our experiments span CIFAR-10, CelebA, FFHQ, and ImageNet-64. QC-FM-Mixture improves over the Baseline on all four after the same number of training epochs and outperforms minibatch OT-CFM on every dataset. These results indicate that projected rank preservation is a promising, low-cost alternative to full mini-batch transport for flow matching.

The contributions of this paper are as follows:
\begin{itemize}
    \item We propose \textbf{Quantile Coupling Flow Matching (QC-FM)}, a lightweight batchwise source construction for flow matching based on projected rank-to-quantile matching and conditional Gaussian completion.
    \item We develop a theoretical account of the QC coupling primitive: for a fixed frame, it eliminates the irreducible variance along the selected slices and makes the ideal per-slice flow straight; we further quantify the transport-cost gain, the effect of direction resampling, the finite-batch approximation, and the dependence among slice codes that separates the training source from the sampling prior.
    \item We introduce two Gaussian-remainder completion schemes, \textbf{QC-FM-Mixture} and \textbf{QC-FM-Adjacency}, which apply QC to an anchor subset, complete the remaining source slots with exact Gaussian samples, and differ in how this Gaussian remainder is paired.
    \item We demonstrate on CIFAR-10, CelebA, FFHQ, and ImageNet-64 that QC-FM-Mixture improves over the Baseline and minibatch OT-CFM, and analyze the anchor-ratio and slice-count trade-offs.
\end{itemize}

\section{Related Work}
 
\paragraph{Flow matching.}
Flow matching \citep{Lip23} trains continuous normalizing flows \citep{Chen18} by regressing the velocity field that generates a prescribed probability path from a source
distribution to the data. Related formulations include rectified flow
\citep{Liu23} and stochastic interpolants \citep{Alb23,Alb25}. In the Baseline
formulation, a Gaussian source and a data sample are drawn
independently and joined by linear interpolation, so the two endpoints are paired
at random.

\paragraph{Couplings by minibatch assignment.}
One line of work replaces this independent pairing with a structured one computed
within each minibatch. Multisample flow matching \citep{Poo23} and OT-CFM
\citep{Ton24} introduce non-trivial minibatch couplings: they draw a noise batch
and a data batch independently, form the $B\times B$ cost matrix between the two
batches of size $B$, and re-pair them by solving an assignment problem ($O(B^3)$ for an
exact solve, or $O(B^2)$ per iteration under entropic
regularization). Subsequent work refines this coupling
\citep{Lin25,Dav25,Li24}. These couplings straighten the flow but inherit the
statistical bias of minibatch OT \citep{Fat20,Boi26}, with benefits emerging
mainly at large batch sizes \citep{Zha25}, though the cost can be amortized
\citep{Mou26}. All share a two-sided design: both endpoints are sampled first, and
the coupling is a matching between them. QC-FM is instead \emph{one-sided}: it
samples only the data batch and \emph{synthesizes} the paired source directly
against a fixed Gaussian quantile grid. Per slice this reduces to the
one-dimensional comonotone coupling (Theorem~\ref{t:gain}), realized by a sort in
$O(B\log B)$; there is no $B\times B$ cost matrix and no batchwise assignment, and
the sampling prior is unchanged.

\paragraph{Data-dependent priors.}
A complementary line makes the source depend on the data rather than matching two
batches. Data-dependent stochastic interpolants \citep{Alb24} condition the source
on the data, but target conditional tasks such as super-resolution and inpainting.
For unconditional generation, the dependence is usually learned via a variational
noise coupling trained jointly with the flow \citep{Sil25}. Closest to QC-FM, two
recent works build the source without learning: \citet{Che26} fit a data-adaptive
prior from one-dimensional quantile functions but still couple through minibatch OT,
and \citet{Mal26} take each image's own low-frequency content as its prior. QC-FM
shares this construct-rather-than-match philosophy but requires no learned component
and no change of the \emph{sampling} prior: nothing is learned or stored for
inference, and generation still starts from the standard Gaussian. Our source is built from within-batch ranks mapped through the Gaussian quantile transform; it equals the standard Gaussian exactly at $k=1$ and, for $k\ge2$,
deviates only through a copula whose transport cost we bound in
Theorem~\ref{t:price}.

\paragraph{Sliced transport and rank-based constructions.}
One-dimensional OT underlies QC-FM's per-slice mechanism: the quadratic
problem is solved in closed form by the comonotone coupling that matches
quantiles \citep{Hoeffding40,Villani09}, which is also the basis of sliced
Wasserstein distances \citep{Rab11,Bon15,Kol19}. Sliced transport plans lift a one-dimensional plan back to the ambient space
and have likewise been used to couple flow-matching batches
\citep{Tan25, Cha25} --- again a matching between two independent batches. The rank-to-Gaussian-quantile map itself
is classical, underlying Gaussianization methods \citep{CheGop00,Lap11,Dai21}. QC-FM
applies this primitive once, batch-locally, rather than composing it into a
standalone generative model. Concurrently, \citet{Gro26} propose a quantile-coupled objective for distributional reinforcement learning, sorting scalar returns
against noise within each batch, whereas QC-FM operates on high-dimensional data
via random orthogonal projections.

\section{Method}

\subsection{Preliminaries: Flow Matching}

Let \(p_0\) denote a simple source distribution, typically \(p_0 = \mathcal{N}(0, I)\), and let \(p_1\) denote the data distribution. FM trains a time-dependent vector field \(v_\theta(x,t)\) to match the velocity of a prescribed conditional probability path between a source sample \(x_0 \sim p_0\) and a data sample \(x_1 \sim p_1\).

The Baseline uses the linear path
$x_t=(1-t)x_0+tx_1$ for $t\in[0,1]$, with target velocity
$u_t=x_1-x_0$.
The model is trained by minimizing
\[
\mathcal{L}_{\mathrm{FM}}(\theta)
=
\mathbb{E}_{t, x_0, x_1}
\left[
\left\|
v_\theta(x_t, t) - u_t
\right\|^2
\right].
\]

Although the objective is simple, the choice of coupling between \(x_0\) and \(x_1\) strongly affects the geometry of the training trajectories. The Baseline is easy to sample but can produce long or conflicting paths, increasing the variance of the target velocity field. Structured couplings can improve this geometry, but existing approaches often rely on solving a mini-batch transport problem. Our goal is to obtain some of the benefits of structured coupling without explicit batchwise assignment.

\subsection{Quantile Coupling Flow Matching}

\subsubsection{Overview}

We propose \textbf{Quantile Coupling Flow Matching (QC-FM)}, a lightweight
one-sided coupling method that preserves the relative order of data samples
along a small number of random projections. Given a data batch
$\mathcal{B}_x=\{x^{(1)},\dots,x^{(B)}\}\subset\R^d$, QC-FM constructs a
corresponding latent batch
$\mathcal{B}_e=\{e^{(1)},\dots,e^{(B)}\}\subset\R^d$ by matching projected data
ranks to Gaussian quantiles. The resulting coupling is structured along the
selected projections and stochastic in the remaining directions.

\subsubsection{Construction}

Let $U=[u_1,\dots,u_k]\in\R^{d\times k}$ have orthonormal projection
directions as columns, sampled for example by QR decomposition of a Gaussian
random matrix.

For each data point $x^{(i)}$ and direction $u_j$, we compute the scalar
projection $v_{ij}=u_j^\top x^{(i)}$.
Within each projection \(j\), we rank the values \(\{v_{1j}, \dots, v_{Bj}\}\) across the batch. 
For a batch \(\mathcal{B}\), let \(r_j(x;\mathcal{B}) \in \{1,\dots,|\mathcal{B}|\}\)
denote the rank of \(u_j^\top x\) within \(\{u_j^\top x' : x' \in \mathcal{B}\}\),
with rank \(1\) the smallest, and write \(r_{ij} := r_j(x^{(i)}; \mathcal{B}_x)\).
We continuity-correct the empirical rank and map it through the inverse
standard normal CDF:
\[
\tau_{ij}=\frac{r_{ij}-0.5}{|\mathcal{B}_x|},
\qquad
z_{ij}=\Phi^{-1}(\tau_{ij}).
\]
Collecting these coordinates as
$z^{(i)}=(z_{i1},\dots,z_{ik})^\top\in\R^k$, we define $e^{(i)}$ as a
standard Gaussian sample conditioned on $U^\top e^{(i)}=z^{(i)}$.

Because \(U\) has orthonormal columns, the conditional distribution admits the closed form
\[
e^{(i)} = U z^{(i)} + (I_d - UU^\top)\epsilon^{(i)},
\qquad
\epsilon^{(i)} \sim \mathcal{N}(0, I_d).
\]
The first term, \(Uz^{(i)}\), enforces the desired projected coordinates, while the second term injects Gaussian variability in the orthogonal complement of \(\mathrm{span}(U)\).

\subsubsection{Geometric Interpretation}

QC-FM can be interpreted as imposing \(k\) linear slice constraints in latent
space. For each direction $u_j$, the condition
$u_j^\top e^{(i)}=z_{ij}$ defines a hyperplane orthogonal to $u_j$.
Therefore, $e^{(i)}$ lies on the intersection of $k$ such hyperplanes. When
$k<d$, this intersection is an affine subspace of dimension $d-k$, and the
residual term $(I_d-UU^\top)\epsilon^{(i)}$ samples from the Gaussian
restricted to that subspace. When $k=d$, the slice constraints fully determine
the latent code as $e^{(i)}=Uz^{(i)}$.

Algorithm~\ref{alg:qcfm} in Appendix~\ref{app:algorithms} summarizes this
construction. We write \(\textrm{QC-FM}(\mathcal B, k)\) for the algorithm, returning
the source samples \(\{e^{(i)}\}\), with \(e^{(i)}\) paired to \(x^{(i)}\),
together with the frame \(U\). Each \(e^{(i)}\) serves as the source endpoint of
the flow-matching path for \(x^{(i)}\), i.e.\ \((x_0, x_1) = (e^{(i)}, x^{(i)})\).

\subsubsection{Interpretation as a Structured Coupling Surrogate}

QC-FM is a sliced, batch-local structured coupling. Along each projection \(u_j\), matching empirical ranks to Gaussian quantiles is equivalent to monotone one-dimensional transport between the projected empirical data distribution and the Gaussian reference. QC-FM combines \(k\) such one-dimensional constraints and constructs a source point that satisfies all of them simultaneously. In this way, the method injects structured geometric information into the coupling without solving a full pairwise transport problem over the batch.

Importantly, QC-FM is still a \emph{batch-local} method: the quantiles are computed from the current mini-batch rather than from the full data distribution. Thus, QC-FM should be understood as a practical surrogate for structured coupling rather than an exact approximation to global OT.

\subsubsection{Properties}

QC-FM has the following useful properties.

\paragraph{Projected rank preservation.}
For every selected direction \(u_j\), the source samples preserve the
batchwise ordering of the data projections:
\[
r_j\big(e^{(i)};\,\mathcal{B}_e\big) = r_j\big(x^{(i)};\,\mathcal{B}_x\big),
\qquad \forall i, j.
\]
This follows because \(u_j^\top e^{(i)} = z_{ij}\) by orthonormality of \(U\), and
\(z_{ij}\) is a strictly increasing transform of \(r_{ij}\), which preserves order.

\paragraph{Gaussian completion in unconstrained directions.}
The residual term $(I_d-UU^\top)\epsilon^{(i)}$ is Gaussian in the orthogonal
complement of the selected projection subspace. Therefore, QC-FM preserves
stochastic variability in directions not explicitly constrained by projected
ranks.

\paragraph{Batchwise Gaussian surrogate rather than exact i.i.d. source sampling.}
At finite batch size, QC-FM does not produce i.i.d.\ draws from
\(\mathcal{N}(0,I_d)\): its coordinates in \(\mathrm{span}(U)\) follow a
deterministic Gaussian quantile grid induced by the batch ranks, while the
orthogonal complement remains Gaussian. At the population level, however, the
distinction is sharper. For \(k=1\), the source marginal is exactly Gaussian;
for \(k\geq2\), any source--prior discrepancy is confined to the copula among
the slice codes (Theorem~\ref{t:exact} in Appendix~\ref{app:mainproofs}).
Thus we interpret finite-batch QC-FM as a structured \emph{batchwise Gaussian
surrogate}, rather than an i.i.d.\ replacement for standard source sampling.

\paragraph{Controllable coupling strength.}
The parameter \(k\) controls how much structure is injected into the coupling. Small \(k\) imposes weak, low-cost constraints, while larger \(k\) enforces projected rank consistency along more directions and produces a stronger geometric bias.

\paragraph{Efficiency.}
QC-FM requires projection, sorting, and orthogonalization, but avoids pairwise cost matrices and transport assignment over the full batch. Its dominant batch-dependent cost is sorting projected values along \(k\) directions, which scales as \(O(kB\log B)\), in contrast to batchwise transport methods that typically require at least quadratic dependence on batch size.

\subsection{Gaussian-Remainder Hybrid Extensions}

The QC coupling above is the primitive analyzed in Section~\ref{sec:theory}.
For $k\ge2$, the QC joint source law can differ from the Gaussian prior because
the slice codes retain data-dependent dependence, even though every individual
slice coordinate is Gaussian in the population construction. We therefore
apply QC only to an anchor subset and complete the remaining source slots with
exact Gaussian samples. We refer to this design as
\emph{Gaussian-remainder completion}. All variants below share the same
source-side completion; they differ only in the joint pairing law assigned to
the Gaussian remainder. Figure~\ref{fig:straightness} later shows how these
coupling choices affect the geometry of learned toy-data trajectories.

Let \(p \in [0,1]\) denote the anchor ratio. Given a data batch
$\mathcal{B}_x=\{x^{(i)}\}_{i=1}^B$, we split
$\mathcal{B}_x=\mathcal{B}_{x,\mathrm{anc}}\cup
\mathcal{B}_{x,\mathrm{rest}}$ with
$|\mathcal{B}_{x,\mathrm{anc}}|=M=\lfloor pB\rfloor$, and write the
corresponding source split as
$\mathcal{B}_e=\mathcal{B}_{e,\mathrm{anc}}\cup
\mathcal{B}_{e,\mathrm{rest}}$. Applying QC-FM yields the latent anchors and
frame $(\mathcal{B}_{e,\mathrm{anc}},U)=
\textrm{QC-FM}(\mathcal{B}_{x,\mathrm{anc}},k)$, together with the anchor pairs
\[
\mathcal{S} = \{(e_m^{\mathrm{anc}}, x_m^{\mathrm{anc}})\}_{m=1}^{M},
\qquad
e_m^{\mathrm{anc}} \in \mathcal{B}_{e,\mathrm{anc}},\ 
x_m^{\mathrm{anc}} \in \mathcal{B}_{x,\mathrm{anc}}.
\]

We define the structured-space distance
$d_U(a,b)=\|U^\top a-U^\top b\|_2$, which measures only the
$k$-dimensional subspace where QC-FM imposes structure. Using $d_U$ prevents
the unconstrained Gaussian complement from dominating neighborhood assignments
in high dimension.

\paragraph{QC-FM-Mixture.}
For the remainder subset, we independently sample
$\widetilde{\mathcal{B}}_{e,\mathrm{rest}}\sim\mathcal{N}(0,I_d)$ and pair
these latents randomly with $\mathcal{B}_{x,\mathrm{rest}}$, exactly as in
the Baseline, giving the random remainder pairs
\[
\mathcal{A} = \{(\tilde e, \sigma_0(\tilde e)) : \tilde e \in \widetilde{\mathcal{B}}_{e,\mathrm{rest}}\},
\]
where \(\sigma_0\) is a uniformly random bijection from
\(\widetilde{\mathcal{B}}_{e,\mathrm{rest}}\) to \(\mathcal{B}_{x,\mathrm{rest}}\).
The final training set is $\mathcal{T}=\mathcal{S}\cup\mathcal{A}$.
Writing $\overline{\pi}_{\mathrm{qc}}:=\E_U[\pi^{\mathrm{qc}}_{k,U}]$ for the frame-averaged QC pair law and $\pi^{\mathrm{ind}}:=\gam d\otimes p_1$ for the Baseline coupling, the population pair law of this construction is
\[
\pi_{\mathrm{mix}}
=
 p\,\overline{\pi}_{\mathrm{qc}}+(1-p)\,\pi^{\mathrm{ind}},
\]
and hence, for every velocity field $v$,
\[
\mathcal{L}^{\mathrm{mix}}_{\mathrm{FM}}(v)
=
p\,\mathcal{L}^{\overline{\pi}_{\mathrm{qc}}}_{\mathrm{FM}}(v)
+(1-p)\,\mathcal{L}^{\pi^{\mathrm{ind}}}_{\mathrm{FM}}(v).
\]
Thus QC-FM-Mixture explicitly retains the Baseline objective on the Gaussian remainder while injecting a structured QC component.

\paragraph{QC-FM-Adjacency.}
QC-FM-Adjacency propagates anchor structure to the remainder of the batch without duplicating targets, so that every data point is used exactly once.

First, each remainder data point $x\in\mathcal{B}_{x,\mathrm{rest}}$ is
assigned to its nearest anchor target in structured coordinates,
$a(x)=\argmin_{m\in\{1,\dots,M\}}d_U(x,x_m^{\mathrm{anc}})$.
This partitions the remainder data into anchor-induced groups
\[
\mathcal{G}_m = \{x \in \mathcal{B}_{x,\mathrm{rest}} : a(x)=m\},
\qquad n_m = |\mathcal{G}_m|.
\]

Next, we sample $|\mathcal{B}_{x,\mathrm{rest}}|$ Gaussian latent points
$\widetilde{\mathcal{B}}_{e,\mathrm{rest}}\sim\mathcal{N}(0,I_d)$.
We then allocate these latent points to anchors with a \emph{parallel auction} that respects the group quotas \(n_m\). In each round, every still-unassigned latent proposes to its nearest anchor (in \(d_U(e,e_m^{\mathrm{anc}})\)) that still has remaining quota; each anchor provisionally keeps its closest proposers up to its quota; rejected latents re-propose in the next round. Since at least one anchor fills its quota per round, the procedure terminates in at most $M$ rounds; in practice it converges in a small, batch-size-independent number of rounds (typically fewer than ten), and it requires $O(MB)$ nearest-anchor distance evaluations. The auction is fully vectorized and \emph{independent of the order in which anchors are processed}. The matching quotas guarantee that every latent and every remainder target is used exactly once, producing latent groups
\[
\mathcal{H}_m \subset \widetilde{\mathcal{B}}_{e,\mathrm{rest}},
\qquad |\mathcal{H}_m| = |\mathcal{G}_m| = n_m.
\]

Finally, latent points in \(\mathcal{H}_m\) are paired with targets
in \(\mathcal{G}_m\) via a bijection \(\sigma_m:\mathcal{H}_m\to\mathcal{G}_m\), chosen at random within each group. The resulting remainder pair set is
\[
\mathcal{N} = \bigcup_{m=1}^{M} \{(e, \sigma_m(e)) : e \in \mathcal{H}_m\}.
\]

This construction uses every target in the batch exactly once and uses the same exact-Gaussian remainder as QC-FM-Mixture. Its distinction is entirely in the pairing: instead of reverting to the Baseline on the remainder, it propagates the local geometry induced by the QC anchors through anchor-centered source and target neighborhoods.

Both hybrids modify only the batchwise source--target pairing; the interpolation
path and flow-matching objective remain unchanged. The QC
construction and hybrid training procedure are given in
Algorithms~\ref{alg:qcfm} and~\ref{alg:hybrid}, respectively.

\section{Theoretical Analysis}\label{sec:theory}

Each training step draws a fresh frame $U$ and regresses against the coupling it induces, and the results below analyze this per-step coupling at the population level. The effect of averaging over redrawn frames is quantified in Appendix~\ref{app:resample}, and the finite-batch approximation in Appendix~\ref{app:consist}.

\begin{proposition}[Complexity]\label{t:complexity}
For a batch of size $B$ in data dimension $d$ with $k$ orthonormal slice directions, one QC-FM batch runs in $O(B\log B)$ time when $d$ and $k$ are held fixed. The detailed full cost is given in Proposition~\ref{app:complexity} (appendix).
\end{proposition}
By contrast, assignment-based couplings such as minibatch OT (OT-CFM \citep{Ton24}) solve a full batch assignment, at $O(B^2d)$ for the cost matrix, plus $O(B^3)$ for an exact Hungarian solve or $O(B^2)$ per Sinkhorn iteration.

\begin{theorem}[The irreducible variance vanishes along slices]\label{t:floor}
For a coupling $\pi$, let $Y:=x_1-x_0$ for the constant velocity of the linear path and
\[
\mathcal{L}^\pi_{\mathrm{FM}}(v):=\E_{t\sim q,\,\pi}\big[\|v(x_t,t)-Y\|^2\big],
\]
where $q$ is a probability density on $[0,1]$ that is positive almost everywhere. Its minimum over all fields is $\inf_v \mathcal{L}^\pi_{\mathrm{FM}}=\int_0^1 q(t)\,\mathcal F^\pi_t\,dt$, where
\[
\mathcal F^\pi_t:=\E_\pi\big[\big\|Y-\E_\pi[Y\mid x_t]\big\|^2\big]
\]
is the \emph{irreducible variance} at time $t$ \citep{Lip23,TimeBlind26}, with slice component $\mathcal F^{\pi}_t(u_j):=\E_\pi\big[(u_j^\top Y-\E_\pi[u_j^\top Y\mid x_t])^2\big]$. Let $\pi^{\mathrm{qc}}_{k,U}$ be the QC-FM coupling for a frame $U$ with $k$ slices and $\pi^{\mathrm{ind}}:=\gam d\otimes p_1$ be the independent coupling used by the Baseline, where $\gam d = \mathcal{N}(0, I_d)$. For a fixed frame $U$,
\begin{equation*}
\begin{split}
\inf_v \mathcal{L}^{\pi^{\mathrm{qc}}_{k,U}}_{\mathrm{FM}}
&\le \inf_v \mathcal{L}^{\pi^{\mathrm{ind}}}_{\mathrm{FM}}
   -\sum_{j=1}^k\int_0^1 q(t)\mathcal F^{\pi^{\mathrm{ind}}}_t(u_j)\,dt\\
&< \inf_v \mathcal{L}^{\pi^{\mathrm{ind}}}_{\mathrm{FM}},
\end{split}
\end{equation*}
the gap being the $q$-weighted sum of the Baseline slice variances, each strictly positive whenever $\Var(u_j^\top X)>0$ for $X \sim p_1$.
\end{theorem}

Here $\E_\pi[Y\mid x_t]$ is the best possible velocity prediction from the current state, and $\mathcal F^\pi_t$ is the mean squared spread of the velocities of all training pairs passing through the point $x_t$ around this prediction; since the network can output only one velocity per point, this spread is the part of the regression error no network can remove.

A lower irreducible variance bounds the attainable loss, not the speed of optimization. We prove no direct variance-to-FID inequality, since FID also depends on the trained network and the sampler.

\begin{theorem}[Slice straightness]\label{t:straight}
For a fixed frame $U$, the ideal QC-FM flow is exactly straight along every slice and distinct pairs never cross within a slice, so all $k$ one-dimensional slice marginals are generated exactly.
\end{theorem}
\noindent The precise assumptions, formal statement, and proof are given in Theorem~\ref{app:straight} (appendix).

Straightness here is not learned but built in. The comonotone slice coupling makes the per-slice velocity single-valued, so the generated law can deviate from $p_1$ only in the dependence among the slice marginals (Theorem~\ref{t:price}).

These results concern the QC primitive and do not bound the floor of the full
hybrid coupling used in the experiments. The hybrids share their source marginal;
Proposition~\ref{t:completion} in Appendix~\ref{app:completion-section} shows
that their exact-Gaussian remainder attenuates whatever prior mismatch the QC
anchors carry.

\section{Experiments}

\subsection{Setup}

\paragraph{Benchmarks and protocol.}
We diagnose coupling geometry on Checkerboard and an 8-Gaussians ring, and
evaluate generation on CIFAR-10~\citep{Krizhevsky09CIFAR},
CelebA-64~\citep{Liu15CelebA}, FFHQ-64~\citep{Karras19FFHQ}, and
class-conditional ImageNet-64~\citep{Deng09ImageNet}. We adopt the EDM-based
image configuration released with
\emph{Flow Matching Guide and Code}
\citep{lipman2024flowmatchingguidecode} as the common model and use its
independent-coupling configuration as the Baseline. Image experiments use the same
U-Net, linear flow-matching path, batch size $64$, Adam
($\beta=(0.9,0.95)$, learning rate $10^{-4}$), EMA, and EDM-style timestep law
across methods; only the coupling changes. Training and timing measurements
use NVIDIA RTX A6000 GPUs, with latency measured on a single GPU.
Class-conditional ImageNet uses classifier-free guidance with scale $1.5$. We
report FID-50K from a Heun solver at $\mathrm{NFE}=50$ after matched training
budgets: $1600/500/1200/60$ epochs for
CIFAR-10/CelebA/FFHQ/ImageNet, respectively. The first three budgets train each
method until performance has sufficiently converged; on ImageNet-64, we instead
compare methods at the same number of training iterations.

\paragraph{Comparisons and configurations.}
We compare the Baseline, exact Hungarian OT-CFM \citep{Ton24}, and the two
QC-FM hybrids. CIFAR-10 ablations select $k=16$, $p=0.2$ for Mixture, and
$p=0.8$ for Adjacency. We transfer these values without per-dataset tuning: the calibrated heuristic of
Appendix~\ref{app:kstar} selects $k=32$, which is used for both reported FFHQ
variants; all other image results use $k=16$. Slice frames are redrawn
independently at every step. We always report the prescribed final checkpoint.

\subsection{Toy Coupling Diagnostics}

Checkerboard tests disconnected support, while 8-Gaussians tests mode-preserving
geometry. We measure the path-length ratio (PLR) and the summed coordinate variance of the
target velocity $x_1-x_0$. Table~\ref{tab:toy_results} reports results after
$20{,}000$ training steps with the same seed; this toy-only selection is not
transferred to images.

\begin{figure}[!t]
  \centering
  \includegraphics[width=\columnwidth]{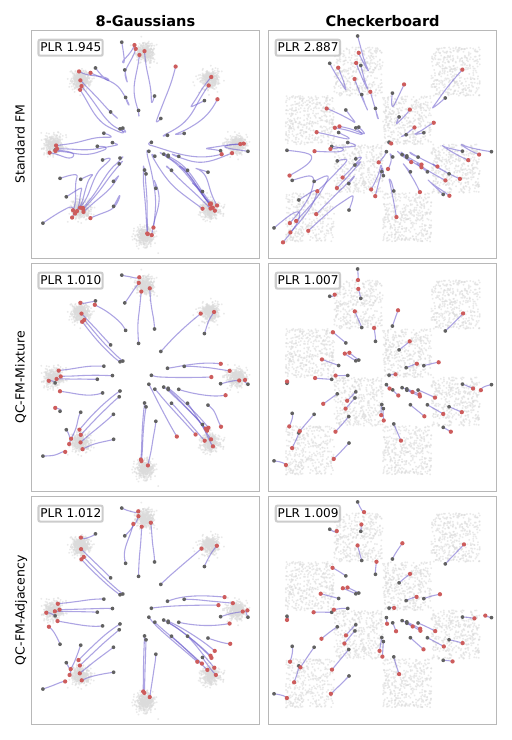}
  \caption{Learned flow trajectories on 8-Gaussians and Checkerboard.
  Light-gray points show the target distribution, dark-gray points are shared
  Gaussian initial states, red points are the corresponding model endpoints,
  and purple curves are $100$-step Euler trajectories. PLR is integrated path
  length divided by endpoint displacement, so $\mathrm{PLR}=1$ corresponds to a straight path.
  Both QC-FM variants visibly straighten trajectories learned by the Baseline.}
  \label{fig:straightness}
\end{figure}

\begin{table}[t]
\centering
\caption{Toy coupling diagnostics after $20{,}000$ steps. Best results are
bold and runners-up are underlined.}
\label{tab:toy_results}
\small\setlength{\tabcolsep}{9pt}
\resizebox{\columnwidth}{!}{%
\begin{tabular}{lcccc}
\toprule
& \multicolumn{2}{c}{Checkerboard} & \multicolumn{2}{c}{8-Gaussians} \\
\cmidrule(lr){2-3}\cmidrule(lr){4-5}
Method & PLR \(\downarrow\) & Velocity Var. \(\downarrow\) &
PLR \(\downarrow\) & Velocity Var. \(\downarrow\) \\
\midrule
Baseline & 2.887 & 4.578 & 1.945 & 5.982 \\
OT-CFM & 1.010 & 0.354 &
1.013 & 0.982 \\
\midrule
QC-FM-Mixture & \textbf{1.007} & \underline{0.174} & \textbf{1.010} & \underline{0.804} \\
QC-FM-Adjacency & \underline{1.009} & \textbf{0.158} &
\underline{1.012} & \textbf{0.707} \\
\bottomrule
\end{tabular}
}
\end{table}

Figure~\ref{fig:straightness} makes the PLR diagnostic tangible: the Baseline
exhibits substantially more curvature, whereas both QC-FM variants yield
nearly straight learned paths. Table~\ref{tab:toy_results} quantifies the same
comparison. QC-FM-Mixture attains the lowest PLR on both datasets, narrowly
improving on OT-CFM, while QC-FM-Adjacency gives the lowest target-velocity
variance. Both QC-FM variants report lower values than OT-CFM on both
diagnostics without solving a full batch assignment. These results probe
coupling mechanics rather than claim global transport optimality. In
particular, Adjacency's lower target-velocity variance does not translate into
lower image FID than Mixture, so neither diagnostic alone determines sample
quality. This distinction is consistent with our theory, which controls
per-slice irreducible variance rather than final sample quality.

\subsection{Image Generation Results}

\begin{table}[t]
\centering
\caption{Main FID-50K comparison at $\mathrm{NFE}=50$. Best results are bold
and runners-up are underlined. Parentheses report
the relative FID change from the Baseline; improvements are
\textcolor{red}{red} and degradations are \textcolor{blue}{blue}. Mixture uses
$(p,k)=(0.2,16)$ except on FFHQ, where it uses $(0.2,32)$; Adjacency uses
$(0.8,16)$ except on FFHQ, where it uses $(0.8,32)$.}
\label{tab:main_results}
\small
\resizebox{\columnwidth}{!}{%
\begin{tabular}{lcccc}
\toprule
& \multicolumn{1}{c}{$32{\times}32$} &
\multicolumn{3}{c}{$64{\times}64$} \\
\cmidrule(lr){2-2}\cmidrule(lr){3-5}
Method &
\shortstack{CIFAR-10} &
\shortstack{CelebA} &
\shortstack{FFHQ} &
\shortstack{ImageNet} \\
\midrule
Baseline & 2.24 & 1.76 & 2.64 & 8.81 \\
OT-CFM (Hungarian) &
2.59\fidworse{15.6} & 2.23\fidworse{26.7} &
3.19\fidworse{20.8} & 8.80\fidbetter{0.1} \\
\midrule
QC-FM-Adjacency (ours) &
\underline{2.12}\fidbetter{5.4} & \underline{1.75}\fidbetter{0.6} &
\underline{2.35}\fidbetter{11.0} & \underline{8.12}\fidbetter{7.8} \\
\textbf{QC-FM-Mixture (ours)} &
\textbf{2.10}\fidbetter{6.3} & \textbf{1.54}\fidbetter{12.5} &
\textbf{2.30}\fidbetter{12.9} & \textbf{7.82}\fidbetter{11.2} \\
\bottomrule
\end{tabular}
}
\vspace{2pt}

\end{table}

Table~\ref{tab:main_results} shows that QC-FM-Mixture is the best method on all
four datasets, improving over the Baseline by
$6.3\%$, $12.5\%$, $12.9\%$, and $11.2\%$ on CIFAR-10, CelebA, FFHQ, and
ImageNet, respectively. It also outperforms Hungarian OT-CFM throughout,
with the largest margins on the two face datasets. Adjacency likewise improves
over the Baseline on all four datasets, including a $7.8\%$ gain on
ImageNet at the transferred $k=16$, although Mixture remains consistently
stronger in image FID. This pattern is consistent with the hybrids' designs:
Mixture retains an explicit Baseline component on its Gaussian
remainder, whereas Adjacency applies QC to most of the batch and organizes the
remainder through anchor-local neighborhoods. As shown in
Table~\ref{tab:toy_results}, the latter attains the lowest toy velocity variance.
An entropic Sinkhorn OT-CFM solver was unstable in this setting (FID above $8$
on CIFAR-10), so we report the stable Hungarian OT-CFM result.
Since minibatch OT can benefit from larger batches \citep{Zha25},
Table~\ref{tab:batch_ablation} in Appendix~\ref{app:batch-size} records the
matched batch-size ablation. Matched-input endpoint comparisons appear in
Appendix~\ref{app:image-qualitative}.

\subsection{Sliced and Sort-Based Couplings}

Table~\ref{tab:sliced_compare} isolates QC-FM's one-sided source construction
by matching the backbone, schedule, batch size, and slice count against
two-sided sliced OT \citep{Tan25} and batchwise sort/quantile coupling
\citep{Gro26}.
At essentially matched step time, sort/quantile coupling and sliced OT have $5.7\%$ and $18.1\%$ higher FID, respectively, than QC-FM-Mixture. The gap therefore cannot be explained by additional computation and supports the utility of constructing the Gaussian source one-sidedly rather than matching two pre-sampled batches.

\begin{table}[t]
\centering
\caption{One- versus two-sided rank/slice couplings on CIFAR-10 under matched
settings. Parentheses report the relative FID increase over QC-FM-Mixture in
\textcolor{blue}{blue}. Time/step is the mean wall-clock time over $100$
training steps after $10$ warm-up steps at $B=64$.}
\label{tab:sliced_compare}
\scriptsize\setlength{\tabcolsep}{2.5pt}
\resizebox{\columnwidth}{!}{%
\begin{tabular}{@{}lcc@{}}
\toprule
Method & FID & Time/step (ms) \\
\midrule
Sliced-OT \citep{Tan25} & 2.48\fidworse{18.1} & 288.5 \\
Sort/quantile \citep{Gro26} & 2.22\fidworse{5.7} & 289.7 \\
\textbf{QC-FM-Mixture (ours)} & \textbf{2.10} & 288.1 \\
\bottomrule
\end{tabular}
}
\end{table}

\subsection{Computational Cost}

Proposition~\ref{t:complexity} predicts $O(B\log B)$ batch dependence for
fixed $(d,k)$, versus $O(B^3)$ for exact Hungarian assignment. Because the
U-Net masks these small differences, we time coupling construction in
isolation on a single GPU.

\begin{table}[t]
\centering
\caption{Isolated CIFAR-10 coupling time (ms/batch; $d=3072,k=16$).
Model forward/backward is excluded.}
\label{tab:overhead}
\small\setlength{\tabcolsep}{5pt}
\resizebox{\columnwidth}{!}{%
\begin{tabular}{lcccc}
\toprule
Coupling & $B{=}64$ & $B{=}256$ & $B{=}1024$ & $B{=}2048$ \\
\midrule
Baseline           & $0.01$ & $0.01$ & $0.02$ & $0.04$ \\
OT-CFM (Sinkhorn)  & $4.95$ & $5.06$ & $6.13$ & $24.62$ \\
OT-CFM (Hungarian) & $0.38$ & $5.09$ & $105.12$ & $568.31$ \\
\midrule
QC-FM-Mixture      & $0.42$ & $0.42$ & $0.50$ & $0.69$ \\
QC-FM-Adjacency    & $2.96$ & $3.89$ & $4.74$ & $5.51$ \\
\bottomrule
\end{tabular}
}
\end{table}

\begin{table}[t]
\centering
\caption{QC-FM cost breakdown at $B=256$, averaged over $200$ calls.
Percentages use the $1079.8$\,ms full training step.}
\label{tab:cost_breakdown}
\small\setlength{\tabcolsep}{4pt}
\resizebox{\columnwidth}{!}{%
\begin{tabular}{lccc}
\toprule
Component & Cost & Time/batch (ms) & \% of step \\
\midrule
Projection $U^\top x$ & $O(Bdk)$ & $0.011$ & $<0.01$ \\
QR / frame generation & $O(dk^2)$ & $0.162$ & $0.02$ \\
$k$ per-slice sorts & $O(kB\log B)$ & $0.074$ & $<0.01$ \\
Adjacency $d_U$ assignment & $O(MB)$ & $2.29$ & $0.21$ \\
\bottomrule
\end{tabular}
}
\end{table}

Table~\ref{tab:overhead} shows that QC-FM-Mixture changes only
$0.42\to0.69$\,ms as $B$ grows from $64$ to $2048$, while Hungarian OT grows
from $0.38$ to $568.31$\,ms. At $B=2048$, Mixture is over $800\times$ faster
than Hungarian and $35\times$ faster than Sinkhorn; Adjacency remains over
$100\times$ faster than Hungarian. Table~\ref{tab:cost_breakdown} shows that at
$B=256$, the shared QC primitives total only $0.25$\,ms and the entire Mixture
coupling is below $0.05\%$ of a step. Adjacency's auction is the only visible
component, yet is still $0.21\%$ of a step.

\subsection{Ablations}

Table~\ref{tab:ablations}(a) shows a non-monotone slice-count trade-off:
Mixture improves from $2.46$ at $k=1$ to $2.10$ at $k=16$, then degrades to
$2.24$ at $k=64$. Additional nondegenerate slices remove corresponding
within-frame residuals and provide the transport-cost gain predicted by
Theorems~\ref{t:floor} and~\ref{t:gain}, but inter-slice dependence
(Theorem~\ref{t:price}), frame disagreement, finite-batch estimation, and
hybrid completion counteract that gain. We therefore use the calibrated
gain--dependence heuristic of Appendix~\ref{app:kstar} as a selection guide.
As shown in Table~\ref{tab:ablations}(b), at $k=16$, Mixture prefers a small
anchor fraction ($p=0.2$), retaining more Baseline pairs, whereas Adjacency
benefits from denser anchor neighborhoods and prefers $p=0.8$. These CIFAR-10
choices define the transferred image configuration above.

\begin{table}[t]
\centering
\caption{CIFAR-10 ablations at the common $1600$-epoch endpoint.
(a) Slice-count sweep for QC-FM-Mixture at $p=0.2$.
(b) Anchor-ratio sweep at $k=16$.}
\label{tab:ablations}
\small
\begin{minipage}[t]{0.48\columnwidth}
\centering
\textbf{(a) Number of slices}\\[2pt]
\resizebox{\linewidth}{!}{%
\begin{tabular}{lcc}
\toprule
Method & \(k\) & FID \\
\midrule
QC-FM-Mixture & 1  & 2.46 \\
QC-FM-Mixture & 4  & 2.42 \\
QC-FM-Mixture & 16 & \textbf{2.10} \\
QC-FM-Mixture & 64 & 2.24 \\
\bottomrule
\end{tabular}
}
\end{minipage}
\hfill
\begin{minipage}[t]{0.48\columnwidth}
\centering
\textbf{(b) Anchor ratio}\\[2pt]
\renewcommand{\arraystretch}{0.78}%
\resizebox{0.88\linewidth}{!}{%
\begin{tabular}{lcc}
\toprule
Method & \(p\) & FID \\
\midrule
QC-FM-Mixture & 0.2 & \textbf{2.10} \\
QC-FM-Mixture & 0.5 & 2.17 \\
QC-FM-Mixture & 0.8 & 2.26 \\
\midrule
QC-FM-Adjacency & 0.2 & 2.56 \\
QC-FM-Adjacency & 0.5 & 2.57 \\
QC-FM-Adjacency & 0.8 & \textbf{2.12} \\
\bottomrule
\end{tabular}
}
\end{minipage}
\end{table}

Table~\ref{tab:pca_fixed} in Appendix~\ref{app:pca-fixed} tests
projection-frame design by comparing the default randomly resampled \(U\) with
a single PCA-aligned \(U\) fixed throughout training. We also reverse the
rank-to-quantile order while preserving the same Gaussian quantile multiset.
Table~\ref{tab:rank_orientation} in Appendix~\ref{app:rank-orientation}
compares the Baseline, anti-monotone, and monotone variants to test whether
correct rank orientation, rather than quantile regularization alone, drives
the gain.

\section{Discussion}

QC-FM should be understood as a \emph{batch-local structured coupling
surrogate}, not an approximation to global OT. Its quantiles are
estimated within each mini-batch, its projections capture only partial
geometry, and its theory controls coupling cost and per-slice regression
variance rather than FID or globally straight trajectories. The
experiments show that this limited, inexpensive bias can improve
sample quality under matched architectures and training budgets. The hybrid
results should likewise be read empirically: Gaussian-remainder completion
attenuates source-prior mismatch, but theory does not rank Mixture and
Adjacency because their remainder pairings define different joint couplings.

Two questions remain open. First, slice and anchor selection is
dataset-dependent. The gain--dependence statistics provide a useful calibrated
guide, but a principled account must also include frame resampling,
finite-batch error, hybrid completion, and their relationship to final sample
quality. Second, QC-FM remains local to the current mini-batch. Replacing
empirical batch quantiles with learned global quantile estimates could yield
more stable slice coordinates across batches; establishing when such estimates
preserve the one-sided coupling guarantees is an important direction.

\section{Conclusion}

We introduced \textbf{Quantile Coupling Flow Matching (QC-FM)}, a lightweight structured coupling method for flow matching. QC-FM constructs source samples by preserving the empirical rank structure of data points along a small number of random projections and completing the remaining dimensions by conditional Gaussian sampling. This yields a simple and efficient alternative to batchwise transport matching that injects geometric structure into the coupling without requiring a full pairwise assignment procedure.

Building on the QC coupling primitive, we proposed Gaussian-remainder completion strategies --- \textbf{QC-FM-Mixture} and \textbf{QC-FM-Adjacency}. Each applies QC to structured anchors and fills the non-anchor source slots with exact Gaussian samples, thereby attenuating source-prior mismatch while retaining the QC bias. Mixture pairs this remainder as in the Baseline, whereas Adjacency organizes it through QC-anchor neighborhoods, using every target exactly once. QC-FM-Mixture is our main variant, delivering the lowest endpoint FID across all four image benchmarks.

Taken together, our results show that preserving projected rank structure is a practical and effective way to bias flow matching toward more organized source--target pairings. While QC-FM is not a substitute for global optimal transport, it offers an attractive middle ground between the Baseline and expensive batchwise transport methods: it is simple to implement, computationally lightweight, and controllable through the number of projections. We believe this makes QC-FM a promising foundation for future work on scalable structured couplings in flow matching.

Future work should make slice and anchor selection more dataset-adaptive and
extend QC-FM from mini-batch empirical quantiles to globally consistent
estimates.

\bibliography{qcfm}


\onecolumn

\begin{center}
{\LARGE\bfseries Supplement Material: One-Sided Quantile Coupling for Flow Matching\par}
\end{center}
\vspace{1em}

\setcounter{secnumdepth}{2} 
\appendix

\setcounter{table}{0}
\renewcommand{\thetable}{\Alph{table}}
\setcounter{figure}{0}
\renewcommand{\thefigure}{\Alph{figure}}

\numberwithin{theorem}{section}
\makeatletter
\renewcommand{\@seccntformat}[1]{Appendix~\csname the#1\endcsname\quad}
\makeatother

\section{Proofs and Further Results}\label{app:mainproofs}

Throughout this appendix, we use the following notation. $X\sim p_1$ is the data variable in $\R^d$, and $\gam n$ is the standard Gaussian on $\R^n$ with $\Phi$ the cdf of $\gam1$. A frame $U=[u_1,\dots,u_k]$ has orthonormal columns, and $P_U^{\perp}:=I_d-UU^\top$ projects onto its complement; a Haar frame is $U$ with uniform orthonormal columns. The QC-FM coupling $\pi^{\mathrm{qc}}_{k,U}$ pairs $X$ with $e=UZ+P_U^{\perp}\epsilon$, $\epsilon\sim\gam d$, where $Z=(Z_1,\dots,Z_k)^\top$ with $Z_j:=\Phi^{-1}(F_{u_j}(u_j^\top X))$ the Gaussian rank code of the slice $u_j^\top X$ (here $F_{u_j}$ is its cdf), and $c(\pi):=\E_{(x_0,x_1)\sim\pi}\|x_1-x_0\|^2$ is the quadratic transport cost. We write $W_2$ for the quadratic Wasserstein distance and $R^2:=\E\|X\|^2$. We also write $\rho_U:=\Law(Z)$ for the joint law of the slice codes and $\alpha_k^U:=\Law(e)$ for the fixed-frame source prior.

\subsection{Population exactness, the copula identity, and soundness}


\begin{theorem}
\label{t:exact}
Assume finite second moments and continuous slice cdfs, and write $\pi^{\mathrm{qc}}_{k,U}=\Law(e,X)$. Then:
\begin{enumerate}[label=\textup{(\alph*)},itemsep=2pt,topsep=2pt]
\item \emph{\textbf{Population Exactness}.} Each $Z_j\sim\gam1$ exactly, and $W:=\Pperp\epsilon\sim N(0,\Pperp)$ is independent of $Z$, so $\Law(e)=\rho_U\otimes\gam{d-k}=\alpha_k^U$ and $\pi^{\mathrm{qc}}_{k,U}$ has marginals $(\alpha_k^U,p_1)$.
\item \emph{\textbf{Copula Identity}.} Each map $V_j\mapsto Z_j$ is nondecreasing with $\Phi(Z_j)=F_j(V_j)$ exactly, so $Z$ and $(V_1,\dots,V_k)$ have the same copula \citep{Nelsen06}, which is unique because the marginals are continuous, by Sklar's theorem \citep{Sklar59}. Hence $e\sim\gam d$ iff the projections are independent. The entire gap between $\alpha_k^U$ and the FM prior is this one $k$-dimensional copula. It leaves the mean of $e$ at zero and every slice coordinate exactly standard Gaussian, while the joint covariance can deviate from the identity, $\Cov(e)=I_d+U\big(\E[ZZ^\top]-I\big)U^\top$, and it vanishes identically at $k=1$.

\item \emph{\textbf{Soundness for $k$}.} For every $k$ and every coupling, the ideal field $v^*$ solves the continuity equation between $\alpha_k^U$ and $p_1$ weakly, using only the tower property. Consequently, under standard well-posedness of the flow, the induced map $\Psi$ transports $\alpha_k^U$ exactly to $p_1$, so correctness depends only on the source marginal $\alpha_k^U$, not on the pairing. Sampling, however, starts from the FM prior $\gam d$ rather than from $\alpha_k^U$. The generated law then deviates from $p_1$ by at most $W_2(\Psi_\#\gam d,p_1)\le \Lip(\Psi)\,W_2(\gam k,\rho_U)$, with $\Lip(\Psi)$ the Lipschitz constant of the flow map. At $k=1$, this deviation vanishes, since $\alpha_1^U=\gam d$ exactly. Integrating from $x_0\sim N(0,I_d)$ then yields samples of law exactly $p_1$ under the stated finite-second-moment, continuous-slice-CDF, and flow well-posedness assumptions. This is the case in which the QC source marginal matches the Gaussian sampling prior exactly, and it anchors the soundness discussion for $k\ge2$.
\end{enumerate}
\end{theorem}

\begin{proof}
\emph{(a).} For $t\in(0,1)$, let $\xi_t:=\sup\{s:F_j(s)\le t\}$; continuity of $F_j$ gives $F_j(\xi_t)=t$ and $\{F_j(V_j)\le t\}=\{V_j\le \xi_t\}$, so $F_j(V_j)\sim\Unif(0,1)$ and $Z_j\sim\gam1$. By rotational invariance of the Gaussian, $W:=\Pperp\epsilon\sim N(0,\Pperp)$, and $W$ is independent of $(X,Z)$ since $\epsilon$ is drawn independently of $X$; hence $\Law(e)=\rho_U\otimes\gam{d-k}=\alpha_k^U$. For $k=1$, the characteristic function factorizes:
\begin{align*}
    \E e^{i\theta^\top e}&=\E e^{i(u^\top\theta)Z}\,\E e^{i\theta^\top W}\\
&=e^{-\frac12(u^\top\theta)^2}e^{-\frac12\|\Pperp\theta\|^2}=e^{-\frac12\|\theta\|^2}.
\end{align*}

\emph{(b).} Each $Z_j$ has continuous strictly increasing cdf $\Phi$, and $\Phi(Z_j)=\Phi(\Phi^{-1}(F_j(V_j)))=F_j(V_j)$, since $\Phi\circ\Phi^{-1}$ is the identity on $(0,1)$ and $F_j(V_j)\in(0,1)$ almost surely. The copula of $Z$ is therefore $\Law(F_1(V_1),\dots,F_k(V_k))$, the copula of $V$, unique by Sklar's theorem for continuous marginals \citep{Sklar59}. The independence copula gives $e\sim\gam d$ by the characteristic-function computation of (a); conversely $e\sim\gam d$ forces $Z\sim\gam k$ with independent components.

\emph{(c).} Let $\varphi\in C_c^\infty$. Since $x_{t+h}-x_t=hY$, the difference quotient is dominated by an integrable bound, $|\varphi(x_{t+h})-\varphi(x_t)|/|h|\le\|\nabla\varphi\|_\infty\|Y\|\in L^1$. Dominated convergence then lets us differentiate under the expectation, and since $\tfrac{d}{dt}x_t=Y$ the chain rule gives $\tfrac{d}{dt}\E[\varphi(x_t)]=\E[\nabla\varphi(x_t)^\top Y]$; the tower property replaces $Y$ by $\E[Y\mid x_t]=v^*(x_t,t)$, yielding
\begin{equation}\label{eq:weakcont}
\frac{d}{dt}\E[\varphi(x_t)]=\E[\nabla\varphi(x_t)^\top v^*(x_t,t)],
\end{equation}
the weak continuity equation between $p_0=\alpha_k^U$ and $p_1$, valid for every $k$ and every coupling. Under standard well-posedness, the induced flow map satisfies $\Psi_\#\alpha_k^U=p_1$ with $\alpha_1^U=\gam d$. Started instead from the FM prior,
\begin{equation}\label{eq:transfer}
\begin{aligned}
W_2(\Psi_\#\gam d,p_1)&=W_2(\Psi_\#\gam d,\Psi_\#\alpha_k) \\ &\le\Lip(\Psi)\,W_2(\gam d,\alpha_k^U) \\ &=\Lip(\Psi)\,W_2(\gam k,\rho_U),
\end{aligned}
\end{equation}
where the last equality holds because the shared factor $\gam{d-k}$ cancels: for product measures $W_2^2(\mu\otimes\nu,\mu'\otimes\nu)=W_2^2(\mu,\mu')$. The bound ``$\le$'' couples $\nu$ to itself by the identity, adding an optimal coupling of $(\mu,\mu')$ at zero extra cost; the reverse ``$\ge$'' holds because any coupling of the two products marginalizes to a coupling of $(\mu,\mu')$, whose cost it can only exceed.
\end{proof}

\subsection{Proof of Proposition~\ref{t:complexity}}


\begin{proposition}[Complexity, full accounting]\label{app:complexity}
For a batch of size $B$ in data dimension $d$ with $k$ orthonormal slice directions, one QC-FM batch costs
\[
O(Bdk)\;+\;O(kB\log B)\;+\;O(dk^2)
\]
time and $O(Bd+dk)$ memory. The three terms are the projection $v_{ij}=u_j^\top x_i$, the $k$ sorts of $B$ values, and the one-off QR orthogonalization of $U$, respectively. For fixed $d$ and $k$ this reduces to $O(B\log B)$, dominated by the sorts: the projection term is linear in $B$ and the orthogonalization term is independent of $B$.
\end{proposition}

\begin{proof}
The $Bk$ projections $v_{ij}=u_j^\top x_i$ cost $O(Bdk)$; the $k$ sorts of $B$ values cost $O(kB\log B)$; QR of a $d\times k$ Gaussian costs $O(dk^2)$; and the complement completion $\epsilon^{(i)}-U(U^\top\epsilon^{(i)})$ costs $O(dk)$ per sample, i.e.\ $O(Bdk)$ over the batch, absorbed into the projection term. Summing gives the stated bound, and holding $d,k$ fixed leaves $O(B\log B)$. No pairwise cost matrix is formed, and the only superlinear dependence on $B$ is the $O(B\log B)$ of the $k$ sorts.
\end{proof}

\subsection{Optimal Slices and Additive Gains}

\begin{theorem}
\label{t:gain}
For a unit direction $u$, define the \emph{per-slice gain constant}
\[
C_u\;:=\;\E[V_uZ_u]\;=\;\int_0^1\Phi^{-1}(\tau)\,F_u^{-1}(\tau)\,d\tau\;\in(0,\sigma_u],
\]
the covariance between a projection $V_u=u^\top X$, with law $u_\#p_1$ and cdf $F_u$, and its Gaussian rank code $Z_u=\Phi^{-1}(F_u(V_u))$. When the slice cdfs are continuous, $C_u$ can be computed from the data before training through
\[
2C_u=\mu_u^2+\sigma_u^2+1-W_2^2\big(u_\#p_1,\gam1\big),
\]
where $\mu_u:=\E[u^\top X]$ and $\sigma_u^2:=\Var(u^\top X)$ are the mean and variance of the slice, with $C_u=\sigma_u$ if and only if the slice $u_\#p_1$ is Gaussian. The following hold.
\begin{enumerate}[label=\textup{(\alph*)},itemsep=2pt,topsep=2pt]
\item For each $j$, the pair $(Z_j,V_j)$ is a.s.\ the comonotone coupling of $(\gam1,(u_j)_\#p_1)$, the unique minimizer of the one-dimensional quadratic transport cost among couplings with these marginals \citep{Hoeffding40,Frechet51,Villani09}, and the complement coordinates are independent Gaussian noise.
\item The batch coupling has cost
\[
c(\pi^{\mathrm{qc}}_{k,U})=R^2+d-2\sum_{j=1}^kC_{u_j},
\]
so relative to a coupling with independent slice codes each direction $u_j$ lowers the transport cost by exactly $2C_{u_j}>0$ (positive whenever $\Var(u_j^\top X)>0$); this is the sense in which $C_{u_j}$ is the gain of slice $j$. The reductions add across slices with no interaction term, and for a Haar-random frame the expected total is $\E_U\big[2\sum_jC_{u_j}\big]=k\big(\tfrac{R^2}{d}+1-SW_2^2(p_1,\gam d)\big)$, where $SW_2^2(p_1,\gam d):=\E_u\,W_2^2(u_\#p_1,\gam1)$ is the sliced-Wasserstein distance (expectation over a uniform direction $u$).
\end{enumerate}
\end{theorem}

\begin{proof}[Proof]
We first record the closed form and bounds for the constant $C_u$ stated before the enumerated claims, then prove (a) and (b).

\emph{The quantile representation of $C_u$.}
Fix a unit direction $u$ and write $V:=u^\top X$, $F:=F_u$ for its cdf, and $Z:=\Phi^{-1}(F(V))$ for its Gaussian rank code. When $F$ is continuous, $\tau:=F(V)$ is uniform on $(0,1)$ (Theorem~\ref{t:exact}(a)); on this event the quantile functions invert the cdfs, so $V=F^{-1}(\tau)$ and $Z=\Phi^{-1}(\tau)$ almost surely. Hence, by the change of variables $V\mapsto\tau=F(V)$ (which sends the law of $V$ to $\Unif(0,1)$),
\begin{equation}\label{eq:Cquantile}
\begin{aligned}
C_u&=\E[VZ]=\E\big[F^{-1}(\tau)\,\Phi^{-1}(\tau)\big] \\ &=\int_0^1 F^{-1}(\tau)\,\Phi^{-1}(\tau)\,d\tau ,
\end{aligned}
\end{equation}
the last equality because $\tau$ is uniform. This is the integral form used below.

\emph{The closed form.}
The one-dimensional Wasserstein distance between laws $\mu,\nu$ on $\R$ with quantile functions $F_\mu^{-1},F_\nu^{-1}$ is
\begin{equation}\label{eq:W2quantile}
W_2^2(\mu,\nu)=\int_0^1\big(F_\mu^{-1}(\tau)-F_\nu^{-1}(\tau)\big)^2\,d\tau ,
\end{equation}
since the comonotone coupling, which pairs equal quantiles, is optimal in one dimension \citep{Villani09}. Taking $\mu=u_\#p_1$ (so $F_\mu^{-1}=F^{-1}$) and $\nu=\gam1$ (so $F_\nu^{-1}=\Phi^{-1}$) and expanding the square,
\begin{equation}\label{eq:W2expand}
\begin{aligned}
W_2^2(u_\#p_1,\gam1)=&\int_0^1(F^{-1})^2\,d\tau-2\int_0^1 F^{-1}\Phi^{-1}\,d\tau  +\int_0^1(\Phi^{-1})^2\,d\tau .
\end{aligned}
\end{equation}
The three integrals are $\int_0^1(F^{-1})^2\,d\tau=\E[V^2]=\mu_u^2+\sigma_u^2$, the cross term $\int_0^1 F^{-1}\Phi^{-1}\,d\tau=C_u$ by \eqref{eq:Cquantile}, and $\int_0^1(\Phi^{-1})^2\,d\tau=\E[\Phi^{-1}(\tau)^2]=1$. Substituting into \eqref{eq:W2expand} and solving for $C_u$,
\begin{equation}\label{eq:Cclosed}
2C_u=\mu_u^2+\sigma_u^2+1-W_2^2(u_\#p_1,\gam1).
\end{equation}

\emph{The bounds $0<C_u\le\sigma_u$.}
Since $\int_0^1\Phi^{-1}(\tau)\,d\tau=0$, we may center the first factor in \eqref{eq:Cquantile}:
\begin{equation}\label{eq:Ccentered}
C_u=\int_0^1\big(F^{-1}(\tau)-\mu_u\big)\Phi^{-1}(\tau)\,d\tau .
\end{equation}
For the lower bound, let $V'$ be an independent copy of $V$ and set $g:=\Phi^{-1}\circ F$, so that $Z=g(V)$ and $g$ is nondecreasing. Since $\E[g(V)]=\E[Z]=\int_0^1\Phi^{-1}=0$, expanding the product $(V-V')(g(V)-g(V'))$ and taking expectations gives
\begin{equation}\label{eq:Cpos}
\begin{aligned}
&\E\big[(V-V')\big(g(V)-g(V')\big)\big]\\
&=2\E[V\,g(V)]-2\E[V]\,\E[g(V)] \\
&=2\E[V\,g(V)]
=2C_u ,
\end{aligned}
\end{equation}
where the first equality uses that $V$ and $V'$ are i.i.d.\ (so $\E[V g(V)]=\E[V'g(V')]$ and $\E[Vg(V')]=\E[V]\E[g(V')]$), and the second uses $\E[g(V)]=0$. Each summand $(V-V')(g(V)-g(V'))$ is nonnegative because $g$ is nondecreasing, so $V-V'$ and $g(V)-g(V')$ always share the same sign; hence $2C_u\ge0$. The inequality is strict when $\Var(V)>0$: then $V\ne V'$ on a positive-probability event, and since $F$ is continuous and $\Phi^{-1}$ strictly increasing, $V\ne V'$ implies $g(V)\ne g(V')$ almost surely, so the nonnegative integrand is strictly positive there, giving $2C_u>0$.

For the upper bound, apply the Cauchy--Schwarz inequality in $L^2(0,1)$ to the two factors of the centered integrand in \eqref{eq:Ccentered}:
\begin{equation}\label{eq:Cub}
\begin{aligned}
C_u&=\int_0^1\big(F^{-1}-\mu_u\big)\Phi^{-1}\,d\tau \\ &\le\Big(\int_0^1(F^{-1}-\mu_u)^2\,d\tau\Big)^{1/2}\Big(\int_0^1(\Phi^{-1})^2\,d\tau\Big)^{1/2} \\ &=\sigma_u\cdot1=\sigma_u ,
\end{aligned}
\end{equation}
using $\int_0^1(F^{-1}-\mu_u)^2\,d\tau=\Var(V)=\sigma_u^2$ and $\int_0^1(\Phi^{-1})^2\,d\tau=1$. Equality in Cauchy--Schwarz forces the two factors to be proportional, $F^{-1}-\mu_u=c\,\Phi^{-1}$ Lebesgue-a.e.\ for some $c\ge0$; both sides are left-continuous and $\Phi^{-1}$ is continuous, so this holds for every $\tau$, i.e.\ $u_\#p_1=N(\mu_u,c^2)$ is Gaussian, and then $c=\sigma_u$ and $C_u=\sigma_u$.

\emph{Proof of (a).}
The identity $F^{-1}(F(v))=v$ holds except on the interiors of the countably many flat intervals of $F$, each carrying zero $u$-slice mass, so $V=S(Z)$ almost surely with $S:=F^{-1}\circ\Phi$ nondecreasing. The pair $(Z,V)$ therefore lies almost surely on a nondecreasing graph, which is the comonotone coupling of $(\gam1,u_\#p_1)$. Among couplings with these marginals, comonotonicity maximizes $\E[ZV]$ (Fr\'echet--Hoeffding \citep{Hoeffding40,Frechet51}), equivalently minimizes $\E[(Z-V)^2]$, the unique optimum of one-dimensional quadratic transport \citep{Villani09}. The complement coordinates $\Pperp\epsilon\sim N(0,\Pperp)$ are independent of $Z$ by construction.

\emph{Proof of (b).}
Write $e=UZ+\Pperp\epsilon$. Expanding the cost,
\begin{equation}\label{eq:cexpand}
c(\pi^{\mathrm{qc}}_{k,U})=\E\|X\|^2+\E\|e\|^2-2\,\E\langle X,e\rangle .
\end{equation}
The middle term uses the orthogonal decomposition $e=UZ+\Pperp\epsilon$: the two parts are orthogonal, so $\|e\|^2=\|UZ\|^2+\|\Pperp\epsilon\|^2=\|Z\|^2+\|\Pperp\epsilon\|^2$ (as $U$ has orthonormal columns), and taking expectations,
\begin{equation}\label{eq:enorm}
\E\|e\|^2=\sum_{j=1}^k\E Z_j^2+\E\|\Pperp\epsilon\|^2=k+(d-k)=d ,
\end{equation}
because each $Z_j\sim\gam1$ gives $\E Z_j^2=1$, and $\Pperp\epsilon\sim N(0,\Pperp)$ with $\Pperp$ a rank-$(d-k)$ projection gives $\E\|\Pperp\epsilon\|^2=\tr\Pperp=d-k$. The sum $\|Z\|^2=\sum_j Z_j^2$ has no cross term $Z_iZ_j$, so \eqref{eq:enorm} holds whatever the joint law of $Z$. The cross term splits along the same decomposition:
\begin{equation}\label{eq:cross}
\begin{aligned}
\E\langle X,e\rangle&=\E\langle X,UZ\rangle+\E\langle X,\Pperp\epsilon\rangle \\ &=\sum_{j=1}^k\E[V_jZ_j]+0=\sum_{j=1}^k C_{u_j} ,
\end{aligned}
\end{equation}
where the first equality is the decomposition $e=UZ+\Pperp\epsilon$, the second expands $\langle X,UZ\rangle=\sum_j (u_j^\top X)Z_j=\sum_j V_jZ_j$ and uses $\E\langle X,\Pperp\epsilon\rangle=0$ (the noise $\epsilon$ is centered and independent of $X$), and the last is the definition $C_{u_j}=\E[V_jZ_j]$. Substituting \eqref{eq:enorm} and \eqref{eq:cross} into \eqref{eq:cexpand} with $\E\|X\|^2=R^2$ gives
\begin{equation}\label{eq:cfinal}
c(\pi^{\mathrm{qc}}_{k,U})=R^2+d-2\sum_{j=1}^k C_{u_j},
\end{equation}
and each $C_{u_j}>0$ whenever $\Var(u_j^\top X)>0$ by \eqref{eq:Cpos}. Because \eqref{eq:enorm} is independent of the joint law of $Z$, the cost reads only the marginals of the slice codes: each slice contributes its own $2C_{u_j}$ with no interaction term, and the dependence among slices --- the copula of Theorem~\ref{t:exact}(b) --- never enters $c(\pi^{\mathrm{qc}}_{k,U})$.

For the Haar average, recall $\mu_u=u^\top m$ and $\sigma_u^2=u^\top\Sigma u$, so $\mu_u^2+\sigma_u^2=u^\top mm^\top u+u^\top\Sigma u=u^\top(mm^\top+\Sigma)u$. Taking $\E_u$ over a uniform direction, where $\E_u[uu^\top]=I/d$,
\begin{equation}\label{eq:haarmoment}
\begin{aligned}
\E_u[\mu_u^2+\sigma_u^2]&=\E_u\big[u^\top(mm^\top+\Sigma)u\big] \\ &=\tr\big((mm^\top+\Sigma)\,\E_u[uu^\top]\big) \\ &=\frac{\tr(mm^\top+\Sigma)}{d}=\frac{R^2}{d},
\end{aligned}
\end{equation}
where the third equality writes the quadratic form as a trace, $u^\top Au=\tr(Auu^\top)$, and the last uses $\tr(mm^\top+\Sigma)=\|m\|^2+\tr\Sigma=\E\|X\|^2=R^2$.
Averaging the closed form \eqref{eq:Cclosed} over $u$ and using the definition of the sliced Wasserstein distance, $\E_u[W_2^2(u_\#p_1,\gam1)]=SW_2^2(p_1,\gam d)$ (with $u_\#\gam d=\gam1$), gives $\E_u[2C_u]=\tfrac{R^2}d+1-SW_2^2(p_1,\gam d)$. Each column of a Haar frame is marginally uniform on $S^{d-1}$, so summing over $j=1,\dots,k$ yields
\begin{equation}\label{eq:haartotal}
\E_U\Big[2\sum_{j=1}^k C_{u_j}\Big]=k\Big(\frac{R^2}{d}+1-SW_2^2(p_1,\gam d)\Big).
\end{equation}
\end{proof}

By (b), the reductions add across slices with no interaction term, so larger $k$ is favorable along the cost axis. The price of raising $k$ lies elsewhere, in the dependence among the slice codes (Theorem~\ref{t:price}). The gain is measurable before training: the closed form \eqref{eq:Cclosed} is estimable from the data by Monte Carlo slicing, and its Haar average \eqref{eq:haartotal} gives the sliced-Wasserstein prediction $k\big(\tfrac{R^2}{d}+1-SW_2^2(p_1,\gam d)\big)$.
\subsection{Proof of Theorem~\ref{t:floor}}

\begin{proof}[Proof of Theorem~\ref{t:floor}]
For any field $v$, the fixed-time regression loss splits into a part that depends on $v$ and a part that does not \citep{Lip23,TimeBlind26}:
\begin{equation}
\begin{aligned}
\ell^\pi_t(v)&=\underbrace{\E_\pi\big[\|v(x_t,t)-\E_\pi[Y\mid x_t]\|^2\big]}_{\text{approximation error}} +\underbrace{\E_\pi\big[\|Y-\E_\pi[Y\mid x_t]\|^2\big]}_{\mathcal F^\pi_t}.
\end{aligned}
\end{equation}
The first term is minimized to zero by the ideal field $v^*_\pi(x,t):=\E_\pi[Y\mid x_t=x]$, so $\mathcal F^\pi_t$ is irreducible: no choice of $v$ lowers it. Averaging under the training-time law gives $\inf_v \mathcal{L}^\pi_{\mathrm{FM}}=\int_0^1 q(t)\,\mathcal F^\pi_t\,dt$.

We use three facts about the QC-FM construction. The code is $Z_j:=\Phi^{-1}(F_j(V_j))$ with $V_j:=u_j^\top X$, the increasing map $S_j:=F_j^{-1}\circ\Phi$ satisfies $V_j=S_j(Z_j)$ a.s., and $m^{(j)}_t:=(1-t)\,\mathrm{id}+t\,S_j$ carries the code to the time-$t$ slice position. We prove the three ingredients and then assemble the loss inequality.

\emph{Step 1: the slice variance vanishes under QC-FM.}
Because the slice cdfs are continuous, $F^{-1}$ is strictly increasing, and hence so are $S_j$ and every $m^{(j)}_t$; each therefore has a well-defined inverse. Along direction $u_j$, the slice position at time $t$ is a deterministic increasing function of the code alone,
\[
u_j^\top x_t=(1-t)\,Z_j+t\,S_j(Z_j)=m^{(j)}_t(Z_j),
\]
which can be inverted to recover the code, $Z_j=(m^{(j)}_t)^{-1}(u_j^\top x_t)$. The slice velocity is likewise a function of the code, $u_j^\top Y=u_j^\top(x_1-x_0)=S_j(Z_j)-Z_j$, so substituting the recovered code expresses it through the single deterministic map
\begin{equation}\label{eq:floortarget}
\begin{gathered}
u_j^\top Y=S_j(Z_j)-Z_j=h^{(j)}_t\big(u_j^\top x_t\big), \\ h^{(j)}_t:=\big(S_j-\mathrm{id}\big)\circ(m^{(j)}_t)^{-1}.
\end{gathered}
\end{equation}
Thus knowing $x_t$ pins down the slice velocity exactly: it is $\sigma(u_j^\top x_t)$-measurable, its conditional variance given $x_t$ is zero, and so the slice variance $\mathcal F^{\pi^{\mathrm{qc}}_{k,U}}_t(u_j)=\E[\Var(u_j^\top Y\mid x_t)]$ is identically zero. The ideal slice field is the explicit map $u_j^\top v^*(x,t)=h^{(j)}_t(u_j^\top x)$.

\emph{Step 2: the same slice variance is strictly positive under the Baseline.}
This is what makes the QC-FM loss \emph{strictly}, not merely weakly, smaller. In the present convention $x_0$ is the noise endpoint and $x_1$ the data endpoint, so under $\pi^{\mathrm{ind}}$ the source is independent of the data and the slice source $Z':=u^\top x_0\sim\gam1$ is independent of $X$. Write $V:=u^\top X$. On the slice, $u^\top x_t=(1-t)Z'+tV$ and $u^\top Y=V-Z'$; eliminating $Z'=(u^\top x_t-tV)/(1-t)$ leaves $u^\top Y=(V-u^\top x_t)/(1-t)$, so conditionally on $x_t$ the only randomness in $u^\top Y$ is that of $V$, and
\begin{equation}\label{eq:floorperp}
\Var(u^\top Y\mid x_t)=\frac{\Var(V\mid x_t)}{(1-t)^2}.
\end{equation}
It therefore suffices that $V$ is not already pinned down by $x_t$, i.e.\ $\Var(V\mid x_t)>0$. Suppose instead $\Var(V\mid x_t)=0$ on a set of positive probability. By Bayes' rule, the conditional law of the data given $x_t$ is the data law $p_1(dx)$ reweighted by the likelihood of the noise $x_0=(x_t-tx)/(1-t)$; since $x_0\sim\gam d=N(0,I_d)$, the noise scaled by $(1-t)$ has density $\phi_{1-t}$, the $N(0,(1-t)^2 I_d)$ density, so the conditional law has density proportional to $\phi_{1-t}(x_t-tx)$ against $p_1(dx)$. This reweighting factor $\phi_{1-t}(x_t-tx)$ is strictly positive and finite for every $x$, so the conditional law and $p_1$ share the same null sets. If $V=u^\top X$ were constant under the conditional law, it would then be constant under $p_1$ as well, contradicting $\Var(u^\top X)>0$. Hence, $\Var(V\mid x_t)>0$ at every interior $t$, and by \eqref{eq:floorperp} the Baseline slice variance is strictly positive there.

To gauge the size of this variance in the simplest case, suppose the data are Gaussian and the slice $V=u^\top X\sim N(\mu_u,\sigma_u^2)$ is independent of the complement coordinates $\Pperp X$, as for isotropic data or $u$ along a principal axis. Image data are not exactly of this form, but the case is analytically tractable and already shows that the variance is of the order of the data's own spread. The complement observation then carries no information about $V$, so $\Var(V\mid x_t)=\Var(V\mid u^\top x_t)$, the pair $(V,u^\top x_t)$ is jointly Gaussian, and
\begin{equation}\label{eq:floorgauss}
\begin{gathered}
\mathcal F^{\pi^{\mathrm{ind}}}_t(u)=\frac{\Var(V\mid x_t)}{(1-t)^2}=\frac{\sigma_u^2}{(\sigma_u^2+1)t^2-2t+1}, \\
\int_0^1 q(t)\,\mathcal F^{\pi^{\mathrm{ind}}}_t(u)\,dt>0;\\
\end{gathered}
\end{equation}
for $q(t)\equiv1$,
\begin{equation}
\begin{gathered}
\int_0^1\mathcal F^{\pi^{\mathrm{ind}}}_t(u)\,dt=\frac\pi2\,\sigma_u ,
\end{gathered}
\end{equation}
where $\sigma_u^2=\Var(u^\top X)$ is the slice variance of the data: a strictly positive amount that QC-FM removes and the Baseline cannot.

\emph{Step 3: the off-slice (complement) variance is no larger under QC-FM.}
The complement component of the irreducible variance is the residual of the velocity components orthogonal to the frame,
\[
\mathcal F^\pi_t(\perp):=\E_\pi\big[\|\Pperp Y-\E_\pi[\Pperp Y\mid x_t]\|^2\big],
\]
where $\Pperp=I-UU^\top$ projects onto the $(d-k)$-dimensional complement of the frame; since $\Pperp$ is a fixed matrix, $\E[\Pperp Y\mid x_t]=\Pperp\E[Y\mid x_t]$. We compare this quantity between the QC-FM coupling $\pi^{\mathrm{qc}}_{k,U}$ and the Baseline coupling $\pi^{\mathrm{ind}}$. Throughout, $\sigma(\cdot)$ denotes the $\sigma$-field generated by the listed random variables --- the information available to the ideal predictor, which is the conditional expectation given that $\sigma$-field.

The idea is simple: both couplings must predict the same complement target from the state $x_t$, and QC-FM's state reveals strictly more about the frame directions than the Baseline does, so its predictor is at least as good and its residual at least as small. We make this precise by realizing both couplings on a single probability space.

Let $X\sim p_1$, a complement noise $W\sim N(0,\Pperp)$, and a frame code $Z''\sim\gam k$ be mutually independent, and let $Z$ be the QC-FM code of $X$, i.e.\ $Z_j=\Phi^{-1}(F_j(u_j^\top X))$. Realize QC-FM through $x_0=UZ+W$ and the Baseline through $\tilde x_0=UZ''+W$. The complement velocity is then
\[
T:=\Pperp Y=\Pperp X-W ,
\]
the same random variable under both couplings, and both states reveal the same complement observation $A_t:=\Pperp x_t=(1-t)W+t\Pperp X$. The two couplings differ only in what their states reveal about the frame directions $V:=U^\top X$.

Under QC-FM, the frame coordinates of the state recover the code exactly, $U^\top x_t=m_t(Z)$ with $m_t$ strictly increasing per slice, and since $V=S(Z)$ with $S$ a bijection this is equivalent to knowing $V$:
\[
\sigma(x_t)=\sigma(V,A_t).
\]
Under the Baseline, the frame coordinates of the state are the noisy mixture $U^\top\tilde x_t=(1-t)Z''+tV$, which reveals $V$ only through the independent noise $Z''$; the state is a function of $(Z'',A_t)$, so
\[
\sigma(\tilde x_t)\subseteq\sigma(V,Z'',A_t).
\]
Refining the conditioning cannot increase the residual: for $\sigma$-fields $\mathcal G\subseteq\mathcal H$, the predictor $\E[T\mid\mathcal G]$ is $\mathcal H$-measurable, whereas $\E[T\mid\mathcal H]$ is the $L^2$-optimal $\mathcal H$-measurable predictor, so $\E\|T-\E[T\mid\mathcal H]\|^2\le\E\|T-\E[T\mid\mathcal G]\|^2$. Apply this with $\mathcal G=\sigma(\tilde x_t)\subseteq\mathcal H=\sigma(V,Z'',A_t)$, and note that the independent code is idle: $Z''\perp(X,W)$ gives $Z''\perp(T,V,A_t)$, so $\E[T\mid V,Z'',A_t]=\E[T\mid V,A_t]$, and $\sigma(V,A_t)=\sigma(x_t)$ is exactly the QC-FM state. Hence
\begin{equation}\label{eq:complement}
\begin{aligned}
\mathcal F^{\pi^{\mathrm{ind}}}_t(\perp)&=\E\big\|T-\E[T\mid\tilde x_t]\big\|^2\ \\ &\ge\ \E\big\|T-\E[T\mid V,A_t]\big\|^2=\mathcal F^{\pi^{\mathrm{qc}}_{k,U}}_t(\perp).
\end{aligned}
\end{equation}

\emph{Assembling the loss inequality.}
The irreducible variance splits along the frame: with $w:=Y-\E_\pi[Y\mid x_t]$, the orthonormal expansion $\|w\|^2=\sum_{j=1}^k(u_j^\top w)^2+\|\Pperp w\|^2$ together with $u_j^\top\E[Y\mid x_t]=\E[u_j^\top Y\mid x_t]$ gives $\mathcal F^\pi_t=\sum_j\mathcal F^\pi_t(u_j)+\mathcal F^\pi_t(\perp)$. Under QC-FM every slice part is zero (Step 1) and the complement part is at most that of the Baseline (Step 3), so for every $t$,
\[
\mathcal F^{\pi^{\mathrm{qc}}_{k,U}}_t=\mathcal F^{\pi^{\mathrm{qc}}_{k,U}}_t(\perp)\le\mathcal F^{\pi^{\mathrm{ind}}}_t(\perp)=\mathcal F^{\pi^{\mathrm{ind}}}_t-\sum_{j=1}^k\mathcal F^{\pi^{\mathrm{ind}}}_t(u_j).
\]
Integrating against $q(t)\,dt$ and using $\inf_v \mathcal{L}^\pi_{\mathrm{FM}}=\int_0^1 q(t)\,\mathcal F^\pi_t\,dt$ gives the displayed bound,
\[
\inf_v \mathcal{L}^{\pi^{\mathrm{qc}}_{k,U}}_{\mathrm{FM}}\le\inf_v \mathcal{L}^{\pi^{\mathrm{ind}}}_{\mathrm{FM}}-\sum_{j=1}^k\int_0^1 q(t)\,\mathcal F^{\pi^{\mathrm{ind}}}_t(u_j)\,dt ,
\]
and the subtracted sum is strictly positive by Step 2, so the right-hand side is strictly below $\inf_v \mathcal{L}^{\pi^{\mathrm{ind}}}_{\mathrm{FM}}$.
\end{proof}

\subsection{Slice Straightness}

\begin{theorem}[Slice straightness, full statement]\label{app:straight}
Fix a frame $U$ and assume each slice $(u_j)_\#p_1$ has a continuous positive density on an interval. Let $v^\ast$ be the ideal QC-FM field of Theorem~\ref{t:floor}, and let $y_t$ be the flow generated by $v^\ast$ from $y_0\sim\gam d$. Along every sliced direction, the ideal flow is exactly straight, and distinct QC-FM pairs never cross in that slice. Consequently, each of the $k$ slice marginals is generated exactly:
\[
u_j^\top y_1 \sim (u_j)_\#p_1,\qquad j=1,\dots,k .
\]
\end{theorem}

\begin{corollary}[Geodesic slice marginals and exact slice generation]\label{a:rk}
In the setting of Theorem~\ref{app:straight}, the slice marginal path is the $W_2$-geodesic from $\gam1$ to $(u_j)_\#p_1$ for each $j$. Consequently, for $y_0\sim\gam d$, every one-dimensional slice marginal of the generated flow is exact, and any deviation of the generated law from $p_1$ is confined to the dependence among slices.
\end{corollary}

\begin{proof}[Proof of Theorem~\ref{app:straight} and Corollary~\ref{a:rk}]
\emph{Straightness.}
Fix a slice direction $u$ and drop the index. With interval support and a continuous positive density, the slice cdf $F$ is a continuously differentiable increasing bijection onto $(0,1)$, so the code-to-quantile map $S=F^{-1}\circ\Phi$ is continuously differentiable with derivative $S'=\phi/(f\circ S)>0$, where $f$ is the slice density and $\phi$ the standard-normal density. Hence $\partial_z m_t=(1-t)+tS'>0$ for the code-to-position map $m_t=(1-t)\,\mathrm{id}+tS$, so $z\mapsto m_t(z)$ is strictly increasing for each $t$ and invertible. The ideal QC-FM field $v^\ast$ of Theorem~\ref{t:floor} acts on this slice as the scalar field $h$ defined below, so proving the claim for $h$ establishes it for $v^\ast$ along $u$.
For a fixed code $z$, the slice position traces the curve $\eta(t)=m_t(z)=(1-t)z+tS(z)$, which is affine in $t$: a straight line from $z$ to $S(z)$ at constant velocity $S(z)-z$. As $z$ ranges over $\R$ these lines fill the strip $[0,1)\times\R$ without ever meeting, so they form a foliation, a partition of the strip into non-crossing curves. Differentiating $\eta(t)$ and substituting $z=m_t^{-1}(\eta)$ recovers the velocity field $h(t,\eta):=(S-\mathrm{id})\big(m_t^{-1}(\eta)\big)$:
\begin{equation}\label{eq:char}
\dot\eta(t)=S(z)-z=(S-\mathrm{id})\big(m_t^{-1}(\eta)\big)=h\big(t,\eta(t)\big).
\end{equation}
Thus each line is a characteristic curve of the field $h$. Conversely $h$ is continuous and, by the implicit function theorem applied to $\partial_z m_t>0$, jointly continuously differentiable in $(t,\eta)$, hence locally Lipschitz in $\eta$; so the scalar ODE $\dot\eta=h(t,\eta)$ has a unique solution through each initial point, which must be the line above. The ideal flow therefore moves each slice coordinate along a straight line, extended to $t=1$ by continuity. Non-crossing across a pair of samples follows from the same monotonicity: if $v_{ij}<v_{i'j}$ then their codes satisfy $z_{ij}<z_{i'j}$, and the two lines $(1-t)z+tv$ keep this order for all $t\in[0,1]$.

\emph{Geodesic marginals.}
Compare the slice marginal at two times $s<t$ by coupling them through the shared code $Z$. The pair $(m_s(Z),m_t(Z))$ moves each code the same signed amount, so it is the comonotone coupling, which is the unique optimal coupling in one dimension \citep{Villani09}. Its cost is
\begin{equation*}
\begin{aligned}
\E\big[(m_t(Z)-m_s(Z))^2\big]&=\E\big[((t-s)(S(Z)-Z))^2\big] \\ &=(t-s)^2\,\E\big[(S(Z)-Z)^2\big] \\&=(t-s)^2\,W_2^2(\gam1,u_\#p_1),
\end{aligned}
\end{equation*}
the last equality because $\E[(S(Z)-Z)^2]$ is exactly the one-dimensional transport cost from $\gam1$ to $u_\#p_1$, attained by the comonotone coupling of $Z\sim\gam1$ with $S(Z)\sim u_\#p_1$. A cost quadratic in $t-s$ means the path moves at constant $W_2$-speed; a constant-speed path joining two measures along optimal couplings is the unique $W_2$-geodesic between them, the displacement interpolation from $\gam1$ to $u_\#p_1$ \citep{McC97}.

\emph{Generated marginals.}
For $y_0\sim\gam d$, the slice coordinate evolves as $u_j^\top y_t=m^{(j)}_t(u_j^\top y_0)$ with $u_j^\top y_0\sim\gam1$, which pushes forward to the geodesic marginal at every $t$. Each one-dimensional slice marginal of the generated flow therefore equals the target $(u_j)_\#p_1$ at $t=1$: it is generated exactly. Since the slice maps are monotone and monotone maps preserve copulas, any deviation of the joint generated law from $p_1$ is carried entirely by the copula of Theorem~\ref{t:exact}(b), i.e.\ by the dependence among slices.
\end{proof}

\subsection{Gaussian-Remainder Completion}\label{app:completion-section}

\begin{proposition}
\label{t:completion}
Let $\bar\alpha_k:=\E_U[\alpha_k^U]$ be the frame-averaged source prior. In the population idealization, any hybrid that uses a QC anchor with probability $p$ and an exact Gaussian source otherwise has source marginal
\[
\mu_{p,k}=p\,\bar\alpha_k+(1-p)\,\gam d ,
\]
and it satisfies
\[
W_2^2(\mu_{p,k},\gam d)\;\le\;p\,W_2^2(\bar\alpha_k,\gam d),
\]
so the deviation contracts by the factor $\sqrt p$ on the $W_2$ scale. Moreover, for a finite batch of size $B$ with $M$ anchors, conditional on the empirical anchor source law $\widehat\alpha_{k,M}$, a uniformly selected source slot has marginal $\tfrac MB\,\widehat\alpha_{k,M}+\big(1-\tfrac MB\big)\gam d$, and the same inequality holds with $p=M/B$ and $\bar\alpha_k$ replaced by $\widehat\alpha_{k,M}$. This guarantee concerns the source marginal only. Mixture and Adjacency share the same Gaussian remainder but induce different source--target pairing laws.
\end{proposition}

\begin{proof}
Let $(A,G)$ be an optimal coupling of $(\bar\alpha_k,\gam d)$ and let $C\sim\operatorname{Bernoulli}(p)$ be independent. Conditional on $C=1$, use the pair $(A,G)$. Conditional on $C=0$, draw $G_0\sim\gam d$ and use the identity pair $(G_0,G_0)$. The first marginal of the resulting coupling is $p\bar\alpha_k+(1-p)\gam d=\mu_{p,k}$, while its second marginal is $\gam d$. Its expected squared cost is
\[
p\,\E\|A-G\|^2=p\,W_2^2(\bar\alpha_k,\gam d).
\]
Since $W_2^2$ is the minimum over all couplings, this construction proves $W_2^2(\mu_{p,k},\gam d)\le pW_2^2(\bar\alpha_k,\gam d)$, and taking square roots gives the $\sqrt p$ contraction. In a finite batch, a uniformly selected source slot is an anchor with probability $M/B$, so the same argument applies conditionally on the empirical anchor law, with $p=M/B$ and $\bar\alpha_k$ replaced by $\widehat\alpha_{k,M}$.
\end{proof}

\subsection{The Price of $k$ under Direction Resampling}

The frame is redrawn at every training step, so the coupling the model sees is the frame average $\bar\pi:=\E_U[\pi^{\mathrm{qc}}_{k,U}]$. The source marginal of $\bar\pi$ is the frame average $\bar\alpha_k:=\E_U[\alpha_k^U]$ of the fixed-frame source priors.

\begin{theorem}
\label{t:price}
By the definition of $W_2$, every coupling $\pi$ of a pair $(\mu,\nu)$ splits as
\[
c(\pi)=W_2^2(\mu,\nu)+E(\pi),\qquad E(\pi)\ge0,
\]
the floor set by the marginals plus the excess of its pairing over OT. Raising $k$ moves the two components in opposite directions, exactly as follows.
\begin{enumerate}[label=\textup{(\alph*)},itemsep=2pt,topsep=2pt]
\item The transport cost still falls linearly in $k$,
\[
c(\bar\pi)=R^2+d-k\Big(\tfrac{R^2}{d}+1-SW_2^2(p_1,\gam d)\Big),
\]
and Theorem \ref{t:floor} and Theorem \ref{app:straight} hold at each step for the frame drawn at that step. This gain is a pure excess reduction: the independent pairing $\pi^\perp:=\bar\alpha_k\otimes p_1$ with the same marginals shares the floor $W_2^2(\bar\alpha_k,p_1)$, so the floor cancels in the difference and $c(\pi^\perp)-c(\bar\pi)=E(\pi^\perp)-E(\bar\pi)=2\bar C\,k$ with $2\bar C:=\tfrac{R^2}d+1-SW_2^2(p_1,\gam d)$.
\item Let $R_Z=\E[ZZ^\top]$ be the correlation matrix of the slice codes and $\kappa_{ij}:=(R_Z)_{ij}$ the code correlation of directions $i,j$. For a Haar frame, $\E_U[\kappa_{ij}^2]$ is the same for all $i\ne j$; denoting this common value $\varrho\in[0,1]$, which depends on $(p_1,d)$ but not on $k$, the dependence obeys the exact identity
\[
\E_U\|R_Z-I\|_F^2=\E_U\sum_{i\ne j}\kappa_{ij}^2=k(k-1)\,\varrho .
\]
\item The dependence of \textup{(b)} is a transport-scale deviation of the source prior $\rho_U$ from the sampling prior $\gam k$: for every frame,
\begin{equation*}
\begin{gathered}
W_2^2(\rho_U,\gam k)\ \ge\ g\big(\|R_Z-I\|_F^2\big),\\ g(F):=\frac{F}{\big(1+\sqrt{1+\sqrt F}\,\big)^2},
\end{gathered}
\end{equation*}
and on average $\E_U\,W_2^2(\rho_U,\gam k)\ge k(k-1)\,\varrho/(1+\sqrt k)^2$. The bound $g$ is quadratic for $F\lesssim1$ and of order $\sqrt F$ for $F\gg1$.
\end{enumerate}
\end{theorem}

The certificate in (c) is per-frame. The frame-averaged source $\bar\alpha_k$ can only be closer to the prior, since $W_2^2(\bar\alpha_k,\gam d)\le\E_U\,W_2^2(\alpha_k^U,\gam d)$ by the joint convexity of $W_2^2$.

\begin{proof}[Proof]
\emph{(a).} The transport cost is linear in the coupling, $c(\pi)=\int\|x_1-x_0\|^2\,d\pi$, so for the averaged coupling $\bar\pi=\E_U[\pi^{\mathrm{qc}}_{k,U}]$, exchanging the frame average with the transport integral by Fubini gives
\begin{equation}\label{eq:cbar}
\begin{aligned}
c(\bar\pi)&=\E_U[c(\pi^{\mathrm{qc}}_{k,U})]=\E_U\Big[R^2+d-2\sum_j C_{u_j}\Big]\\ &=R^2+d-k\Big(\tfrac{R^2}{d}+1-SW_2^2(p_1,\gam d)\Big),
\end{aligned}
\end{equation}
substituting the per-frame cost identity of Theorem~\ref{t:gain}(b) and averaging with the Haar formula \eqref{eq:haartotal}. For the excess-reduction claim, the independent pairing $\pi^\perp$ with the same marginals $(\bar\alpha_k,p_1)$ has cost $c(\pi^\perp)=\E\|e\|^2+\E\|X\|^2=d+R^2$ (the cross term vanishes by independence and $\E[e]=0$), and both couplings decompose as $c=W_2^2(\bar\alpha_k,p_1)+E$ with the same floor; hence the floor cancels in the difference and $c(\pi^\perp)-c(\bar\pi)=E(\pi^\perp)-E(\bar\pi)=2\bar C\,k$, a pure reduction of the excess component.
The per-frame statements of Theorems~\ref{t:floor} and \ref{app:straight} apply conditionally on the drawn $U$ and describe the coupling used at that step. The network itself does not observe $U$, and the floor of its averaged objective decomposes as in Proposition~\ref{a:resample}.

\emph{(b).}
The Frobenius energy carries two nested expectations,
\[
\E_U\|R_Z-I\|_F^2=\E_U\sum_{i\ne j}\big(\E[Z_iZ_j]\big)^2 ,
\]
an inner expectation over the data inside each $\kappa_{ij}=\E[Z_iZ_j]$, and an outer one over the frame $U$. The inner one is already performed, so $\kappa_{ij}$ is a deterministic function of the pair $(u_i,u_j)$ and only $\E_U$ remains. By linearity, it passes inside the off-diagonal sum, $\E_U\|R_Z-I\|_F^2=\sum_{i\ne j}\E_U[\kappa_{ij}^2]$; the standing continuity of the slice cdfs makes each code exactly $N(0,1)$, so $R_Z$ has unit diagonal and $\|R_Z-I\|_F^2=\sum_{i\ne j}\kappa_{ij}^2$ per frame. The Haar law of $U$ is invariant under permutation of its columns, so the columns are exchangeable --- their joint law is unchanged by any permutation --- and $\E_U[\kappa_{ij}^2]$ therefore takes the same value for every ordered pair $i\ne j$. Moreover, any two columns of a Haar $k$-frame have the joint law of the first two columns of a Haar $d\times d$ orthogonal matrix, which does not involve $k$; hence that common value is a function of $(p_1,d)$ alone. Summing the $k(k-1)$ equal terms yields the exact identity
\begin{equation}\label{eq:kappaid}
\begin{gathered}
\E_U\|R_Z-I\|_F^2=\sum_{i\ne j}\E_U[\kappa_{ij}^2]=k(k-1)\,\varrho,\\ \varrho:=\E_U[\kappa_{ij}^2]\in[0,1],
\end{gathered}
\end{equation}
where $\varrho\in[0,1]$ because $\kappa_{ij}^2\le1$ by Cauchy--Schwarz for the unit-variance codes. 
Finally, $\varrho=0$ exactly when the codes are pairwise uncorrelated for Haar-almost-every orthogonal direction pair. For isotropic Gaussian data $p_1=N(0,\sigma^2I_d)$, projections onto orthogonal directions are independent, so every $\kappa_{ij}=0$ and the dependence vanishes identically. In general, $\|R_Z-I\|_F^2$ is a covariance-based certificate of prior mismatch rather than a complete measure of dependence. The constant $\varrho$ is a pre-training data statistic.

\emph{(c).}
This part connects the dependence of (b) to sampling: by Theorem~\ref{t:exact}(c), a flow trained on the source $\alpha_k^U$ but sampled from the FM prior incurs a generation error controlled by $W_2(\gam k,\rho_U)$, so a lower bound on this deviation in terms of the dependence certifies that the dependence is a genuine transport-scale price. The shared factor $\gam{d-k}$ cancels as in \eqref{eq:transfer}, giving $W_2^2(\alpha_k^U,\gam d)=W_2^2(\rho_U,\gam k)$, so it suffices to bound the latter.

\emph{Step 1: reduce to the correlation matrix.}
Gelbrich's inequality \citep[Thm.~2.1]{Gelbrich90} bounds the $W_2$ distance between any two laws from below by the $W_2$ distance between Gaussians matching their means and covariances. For centered laws that Gaussian distance is the Bures distance between the covariances, $\Bures^2(A,B):=\tr\big(A+B-2(A^{1/2}BA^{1/2})^{1/2}\big)$. The codes have covariance $R_Z$ with unit diagonal, so applying the bound to $R_Z$ and $I$,
\begin{equation}\label{eq:bures}
\begin{aligned}
W_2^2(\rho_U,\gam k)\ &\ge\ \Bures^2(R_Z,I) \\&=\tr\big(R_Z+I-2R_Z^{1/2}\big) \\ &=2\sum_{i=1}^k\big(1-\sqrt{\lambda_i}\big),
\end{aligned}
\end{equation}
where $\lambda_1,\dots,\lambda_k$ are the eigenvalues of $R_Z$. Since $\sum_i\lambda_i=\tr R_Z=k$, the last expression equals $\sum_i\big(\sqrt{\lambda_i}-1\big)^2$, which the next step bounds below.

\emph{Step 2: transport lower bounds.}
Let $F:=\|R_Z-I\|_F^2$; by the unit diagonal, $F=\sum_i(\lambda_i-1)^2=\sum_{i\ne j}\kappa_{ij}^2$, the total squared code correlation of the frame. For each eigenvalue, multiplying and dividing by $(1+\sqrt{\lambda_i})^2$ gives $(\sqrt{\lambda_i}-1)^2=(\lambda_i-1)^2/(1+\sqrt{\lambda_i})^2$, whence, per frame,
\begin{equation}\label{eq:PSchain}
\begin{gathered}
W_2^2(\rho_U,\gam k)\ \ge\ \sum_i\frac{(\lambda_i-1)^2}{\big(1+\sqrt{\lambda_i}\big)^2}\ \ge\ \frac{F}{\big(1+\sqrt{\lambda_{\max}}\big)^2},\\ \lambda_{\max}:=\lambda_{\max}(R_Z).
\end{gathered}
\end{equation}
Two unconditional bounds on $\lambda_{\max}$ make \eqref{eq:PSchain} usable. First, $R_Z$ is a covariance matrix, so its eigenvalues are nonnegative and sum to $\tr R_Z=k$; a nonnegative family summing to $k$ has every member at most $k$, so $\lambda_{\max}\le k$ and \eqref{eq:PSchain} gives, per frame, $W_2^2(\rho_U,\gam k)\ge F/(1+\sqrt k)^2$. This coefficient is deterministic and the bound is linear in $F$. Averaging therefore passes through it, and \eqref{eq:kappaid} with $\E_UF=k(k-1)\varrho$ gives
\[
\E_U\,W_2^2(\rho_U,\gam k)\ \ge\ \frac{k(k-1)\,\varrho}{(1+\sqrt k)^2}.
\]
Second, with $\|\cdot\|_{\mathrm{op}}$ the largest singular value, the top eigenvalue's deviation satisfies $\lambda_{\max}-1\le\|R_Z-I\|_{\mathrm{op}}\le\|R_Z-I\|_F=\sqrt F$. Substituting $\lambda_{\max}\le1+\sqrt F$ into the denominator of \eqref{eq:PSchain} gives the self-normalizing per-frame bound
\begin{equation}\label{eq:selfnorm}
W_2^2(\rho_U,\gam k)\ \ge\ g(F),\qquad g(F)=\frac{F}{\big(1+\sqrt{1+\sqrt F}\,\big)^2}.
\end{equation}
This bound is stated per frame because $g$ is concave and so, unlike the linear bound above, does not pass through $\E_U$. Its two regimes locate where the transport image of the dependence stops growing quadratically. As $F\to0$, $g(F)=\tfrac14F+O(F^{5/4})$, the constant $\tfrac14$ arising as $(1+\sqrt\lambda)^{-2}$ at $\lambda=1$: for small dependence the prior's transport deviation grows at the full quadratic rate. As $F\to\infty$, $g(F)=\sqrt F+O(F^{1/4})$: the transport deviation grows only like $\sqrt F$, so the quadratic growth of the price is carried by the identity \eqref{eq:kappaid}, not by its transport image.
\end{proof}

\subsection{Variance Decomposition under Direction Resampling}\label{app:resample}

In training, the frame is redrawn at every step and the velocity network $v(x_t,t)$ does not observe $U$. All expectations below are taken on the joint space where $U$ is a Haar frame and $(x_0,x_1)\sim\pi^{\mathrm{qc}}_{k,U}$ given $U$, so the pair law is the frame average $\bar\pi=\E_U[\pi^{\mathrm{qc}}_{k,U}]$ of the preceding subsection, and the attainable floor of the network is $\inf_v \mathcal{L}^{\bar\pi}_{\mathrm{FM}}=\int_0^1 q(t)\,\mathcal F^{\bar\pi}_t\,dt$ in the notation of Theorem~\ref{t:floor}. The next results separate what survives the averaging from what the resampling adds, and then bound the added term.

\begin{proposition}[Variance decomposition under direction resampling]\label{a:resample}
In the setting above,
\begin{equation}\label{eq:floordecomp}
\begin{gathered}
\mathcal F^{\bar\pi}_t\;=\;\E_U\big[\mathcal F^{\pi^{\mathrm{qc}}_{k,U}}_t\big]\;+\;D_t, \\
D_t\;:=\;\E\big[\big\|\E[Y\mid x_t,U]-\E[Y\mid x_t]\big\|^2\big]\;\ge\;0 .
\end{gathered}
\end{equation}
\end{proposition}
\begin{proof}
Conditioning on $(x_t,U)$ refines conditioning on $x_t$ alone, $\sigma(x_t)\subseteq\sigma(x_t,U)$. Decompose $Y-\E[Y\mid x_t]=\big(Y-\E[Y\mid x_t,U]\big)+\big(\E[Y\mid x_t,U]-\E[Y\mid x_t]\big)$. The first residual is orthogonal in $L^2$ to every $\sigma(x_t,U)$-measurable variable, in particular to the second, so the cross term vanishes in expectation. Given $U$, the conditional $\E[Y\mid x_t,U]$ is the per-frame ideal field of the coupling $\pi^{\mathrm{qc}}_{k,U}$, so the first squared term averages to $\E_U\big[\mathcal F^{\pi^{\mathrm{qc}}_{k,U}}_t\big]$, which is \eqref{eq:floordecomp}.
\end{proof}

\begin{proposition}[General bounds on the resampling term]\label{a:Dbound}
Let $m:=\E X$. For every $t\in[0,1]$,
\begin{equation}\label{eq:Dchain}
0\;\le\;D_t\;\le\;R^2+d-2\bar C\,k-\|m\|^2,
\end{equation}
and the time integral is exactly the excess of the $U$-blind floor over the mean per-frame floor,
\begin{equation}\label{eq:Dprice}
\int_0^1 q(t)\,D_t\,dt\;=\;\inf_v \mathcal{L}^{\bar\pi}_{\mathrm{FM}}\;-\;\E_U\Big[\inf_v \mathcal{L}^{\pi^{\mathrm{qc}}_{k,U}}_{\mathrm{FM}}\Big].
\end{equation}
\end{proposition}
\begin{proof}
The lower bound is the definition of $D_t$ as a mean squared norm. For the upper bound, \eqref{eq:floordecomp} gives $D_t\le\mathcal F^{\bar\pi}_t$, since the remaining term is nonnegative. The conditional mean minimizes the $L^2$ prediction error over all functions of $x_t$, in particular beating the constant $\E Y$, so $\mathcal F^{\bar\pi}_t\le\E\|Y-\E Y\|^2=\E\|Y\|^2-\|\E Y\|^2$. Here $\E Y=\E X-\E e=m$, since $\E Z=0$ and $\E\epsilon=0$ give $\E e=0$, and $\E\|Y\|^2=c(\bar\pi)=R^2+d-2\bar C\,k$ by Theorem~\ref{t:price}(a), which yields \eqref{eq:Dchain}. Retaining the subtracted term instead gives the sharper form $D_t\le R^2+d-2\bar C\,k-\|m\|^2-\E_U\big[\mathcal F^{\pi^{\mathrm{qc}}_{k,U}}_t\big]$. For \eqref{eq:Dprice}, integrate \eqref{eq:floordecomp} against $q(t)\,dt$. Tonelli exchanges $\E_U$ with the weighted time integral, and $\int_0^1 q(t)\,\mathcal F^{\pi}_t\,dt=\inf_v \mathcal{L}^{\pi}_{\mathrm{FM}}$ for $\bar\pi$ and for each $\pi^{\mathrm{qc}}_{k,U}$ by the characterization in Theorem~\ref{t:floor}.
\end{proof}

\begin{proposition}[Isotropic direction-mixture price]\label{a:Dt}
For the special case of isotropic Gaussian data $p_1=N(0,\sigma^2 I_d)$ with $\sigma\ne1$, let $U$ be a Haar frame of rank $k\le d$ and take linear paths. Write $s_\star:=(1-t)+t\sigma$ and $s_\perp^2:=(1-t)^2+t^2\sigma^2$ for the in-slice and complement interpolation scales, and set $\Delta:=s_\perp^{-2}-s_\star^{-2}$ and $\lambda:=\tfrac{\Delta}{2}\|x_t\|^2$. Then the slice dependence of Theorem~\ref{t:price}(b) vanishes identically, i.e.\ $\alpha_k^U=\gam d$ for every $U$ (the fixed-frame source prior coincides with the FM prior), so $D_t$ is the only price of resampling, and the following hold.
\begin{enumerate}[label=\textup{(\alph*)},itemsep=2pt,topsep=2pt]
\item The resampled field is linear, 
\[\E[Y\mid x_t,U]=[a_\perp I+(a_\star-a_\perp)UU^\top]x_t\] with $a_\star=\tfrac{\sigma-1}{s_\star}$ and $a_\perp=\tfrac{t\sigma^2-(1-t)}{s_\perp^2}$, and
\begin{equation*}
\begin{gathered}
D_t=(a_\star-a_\perp)^2\,\E\big[\|x_t\|^2\,\alpha(1-\alpha)\big],\\ \|x_t\|^2\overset d=s_\star^2\chi^2_k+s_\perp^2\chi^2_{d-k},
\end{gathered}
\end{equation*}
where $\chi^2_k$ and $\chi^2_{d-k}$ are independent chi-squared variables with $k$ and $d-k$ degrees of freedom and $\alpha$ is the tilted-Beta mean of \eqref{eq:alpha} below.
\item At the crossover time $t^\dagger=\tfrac1{1+\sigma}$, where the source and target scales balance, $1-t=t\sigma$, the field becomes frame-free and $D_{t^\dagger}=0$; at the endpoints $D_0=\sigma^2k(1-\tfrac kd)$ and $D_1=k(1-\tfrac kd)$, and $D_t\equiv0$ when $k=d$.
\item The resampled irreducible variance strictly beats the Baseline coupling at every $k\le d$ and every interior $t$,
\[
k\frac{\sigma^2}{s_\perp^2}-D_t\ \ge\ k\,\frac{2\sigma^3 t(1-t)}{s_\perp^4}\ >\ 0 .
\]
\end{enumerate}
\end{proposition}
\begin{proof}
\emph{(a).}
For isotropic Gaussian data the population codes are linear, $Z=U^\top X/\sigma$, so the per-slice regressions of Theorem~\ref{t:floor}, specialized to a Gaussian slice, make $\E[X\mid x_t,U]$ linear with the stated in-frame and complement slopes. Hence
\begin{equation}\label{eq:Dtfield}
\E[Y\mid x_t,U]=\big[a_\perp I+(a_\star-a_\perp)P_U\big]x_t,\quad P_U=UU^\top,
\end{equation}
and $D_t=(a_\star-a_\perp)^2\,\E\big[\|P_Ux_t-\E[P_Ux_t\mid x_t]\|^2\big]$. Conditioned on $x_t$, the law of $U$ is the rank-one Bingham tilt of the Haar measure, with density proportional to $\exp(\tfrac\Delta2 x_t^\top P_U x_t)=\exp(\lambda\beta)$, where $\beta:=\|U^\top\hat x_t\|^2$ and $\hat x_t:=x_t/\|x_t\|$. Under the Haar prior, $\beta\overset d=\chi^2_k/(\chi^2_k+\chi^2_{d-k})\sim\mathrm{Beta}(\tfrac k2,\tfrac{d-k}2)$ by rotational invariance, so the conditional law of $\beta$ given $x_t$ is the exponential tilt of this Beta law, 
whose mean is the log-derivative of the Beta moment generating function $\E[e^{\lambda\beta}]={}_1F_1(\tfrac k2;\tfrac d2;\lambda)$, Kummer's function $M(\tfrac k2,\tfrac d2,\lambda)$, given by the integral representation of \citet[Eq.~13.2.1]{AS64}. Differentiating with the derivative formula \citep[Eq.~13.4.8]{AS64} yields
\begin{equation}\label{eq:alpha} 
\alpha=\partial_\lambda\log{}_1F_1\!\big(\tfrac k2;\tfrac d2;\lambda\big)=\frac kd\,\frac{{}_1F_1(\tfrac k2{+}1;\tfrac d2{+}1;\lambda)}{{}_1F_1(\tfrac k2;\tfrac d2;\lambda)} .
\end{equation}
By symmetry of the tilt around the axis $\hat x_t$, $\E[P_Ux_t\mid x_t]=\alpha\,x_t$ with $\alpha:=\E[\beta\mid x_t]$. Since $\|P_Ux_t\|^2=x_t^\top P_Ux_t=\|x_t\|^2\beta$, expanding the square gives $\E\big[\|P_Ux_t-\alpha x_t\|^2\,\big|\,x_t\big]=\|x_t\|^2\alpha(1-\alpha)$, and taking the outer expectation yields the displayed $D_t$.

\emph{(b).}
Direct simplification gives $a_\star-a_\perp=\sigma(1-(1+\sigma)t)/(s_\star s_\perp^2)$, which vanishes exactly at $t^\dagger=\tfrac1{1+\sigma}$, so $D_{t^\dagger}=0$. At $t^\dagger$ the field \eqref{eq:Dtfield} reduces to $a_\perp x_t$, which no longer involves $U$, so the per-frame ideal fields agree and the disagreement vanishes. The endpoints follow from $\lambda\to0$, where $\alpha\to k/d$.

\emph{(c).}
The key identity $\E[\|x_t\|^2\alpha]=ks_\star^2$ follows from the tower property. Since $\alpha=\E[\beta\mid x_t]$,
\begin{equation}\label{eq:towerid}
\begin{aligned}
\E\big[\|x_t\|^2\alpha\big]&=\E\big[\|x_t\|^2\E[\beta\mid x_t]\big]=\E\big[\|x_t\|^2\beta\big]\\&=\E\big[\|U^\top x_t\|^2\big]=\E\big[s_\star^2\|Z\|^2\big]=k\,s_\star^2,
\end{aligned}
\end{equation}
where the first equality substitutes $\alpha=\E[\beta\mid x_t]$, the second is the tower property $\E[\|x_t\|^2\E[\beta\mid x_t]]=\E[\|x_t\|^2\beta]$ (valid because $\|x_t\|^2$ is $\sigma(x_t)$-measurable), the third uses $\|x_t\|^2\beta=\|x_t\|^2\|U^\top\hat x_t\|^2=\|U^\top x_t\|^2$, the fourth uses $U^\top x_t=(1-t)Z+tU^\top X=s_\star Z$ (the slice pairs are perfectly correlated for isotropic Gaussian data), and the last uses $\E\|Z\|^2=k$. Since $\alpha\in(0,1)$, $\alpha(1-\alpha)\le\alpha$ pointwise, so $\E[\|x_t\|^2\alpha(1-\alpha)]\le ks_\star^2$ and $D_t\le k(a_\star-a_\perp)^2 s_\star^2$. Comparing with the per-slice Baseline variance $\sigma^2/s_\perp^2$ of Theorem~\ref{t:floor},
\begin{equation}\label{eq:margin}
k\frac{\sigma^2}{s_\perp^2}-k(a_\star-a_\perp)^2 s_\star^2
=\frac{k\sigma^2}{s_\perp^4}\Big[s_\perp^2-(1-(1+\sigma)t)^2\Big],
\end{equation}
and expanding the bracket,
\begin{equation}\label{eq:bracket}
\begin{aligned}
&s_\perp^2-(1-(1+\sigma)t)^2\\&=\big[(1-t)^2+t^2\sigma^2\big]-\big[1-2(1+\sigma)t+(1+\sigma)^2t^2\big]\\&=2\sigma t(1-t)>0
\end{aligned}
\end{equation}
on $(0,1)$, which gives the margin $2\sigma^3 t(1-t)/s_\perp^4>0$.
\end{proof}

By Theorem~\ref{t:floor}, given $U$ the slice variances of $\pi^{\mathrm{qc}}_{k,U}$ vanish, so the first term of \eqref{eq:floordecomp} reduces to the complement variance and the slice-variance elimination survives the frame averaging. The resampling enters only through $D_t$, the mean squared disagreement of the per-frame ideal fields around their frame average, so per-frame slice-variance elimination does not by itself transfer to the $U$-blind model, whose floor gains this nonnegative term. Propositions~\ref{a:Dbound} and \ref{a:Dt} quantify the term. In general, $D_t$ obeys the moment bound \eqref{eq:Dchain}, and its time integral equals the excess of the $U$-blind floor over the mean per-frame floor \eqref{eq:Dprice}. In the isotropic Gaussian model, $D_t$ is given analytically, vanishes at an interior crossover time, equals $\sigma^2k(1-\tfrac kd)$ and $k(1-\tfrac kd)$ at the endpoints, and leaves the resampled floor strictly below the Baseline floor at every $k$, providing a concrete guide to its size. The fixed-frame results of Theorems~\ref{t:floor} and \ref{app:straight} describe the mechanism of the coupling used at each training step, and \eqref{eq:floordecomp}, together with these bounds, quantifies what the deployed model inherits from them.

\subsection{Consistency of the Batch Scheme}\label{app:consist}

The population theory above uses the slice cdfs, while the algorithm uses within-batch ranks. This subsection quantifies the gap. Given a batch $\{x^{(i)}\}_{i\le B}$ and a frame $U$, let $r_{ij}$ be the rank of the projection $v_{ij}=u_j^\top x^{(i)}$ within slice $j$ (projections assumed distinct within each slice), and let the batch codes be $z_{ij}:=g_{r_{ij}}$, where $g_r:=\Phi^{-1}\big(\tfrac{r-1/2}{B}\big)$ is the midpoint quantile grid. Write $\hat G_B:=\frac1B\sum_{r\le B}\delta_{g_r}$ for the grid measure and its cdf, and $s_B:=\frac1B\sum_{r\le B}g_r^2$ for its second moment. Deviations are measured by the Kolmogorov--Smirnov distance $\KS(\mu,\nu):=\sup_{x\in\R}|F_\mu(x)-F_\nu(x)|$, with $F_\mu,F_\nu$ the cdfs, and by the bounded-Lipschitz distance $\BL(\mu,\nu):=\sup\big\{\int f\,d\mu-\int f\,d\nu:\|f\|_\infty\le1,\ \Lip(f)\le1\big\}$.

The first result states what the batch source actually is. It is not an i.i.d.\ Gaussian sample but a deterministic stratification of one, so its deviation from the population Gaussian is exact rather than asymptotic, and its slice statistics carry no sampling noise.

\begin{proposition}[The finite-$B$ source is an exact stratification]\label{a:grid}
Assume the projections are distinct within each slice. Then, within any batch, the slice codes $\{z_{ij}\}_{i\le B}$ form a permutation of the fixed grid $\{g_r\}_{r\le B}$, so QC-FM stratifies its slice coordinates rather than sampling them. Consequently,
\begin{enumerate}[label=\textup{(\alph*)},itemsep=2pt,topsep=2pt]
\item the grid deviates from the Gaussian by $\KS(\hat G_B,\gam1)=\frac1{2B}$ exactly, and
\item every slice statistic is deterministic, $\frac1B\sum_if(z_{ij})=\frac1B\sum_rf(g_r)$ for every $f$, with mean exactly $0$ and second moment exactly $s_B$. Under the Baseline, the same statistic is computed from i.i.d.\ Gaussian codes and has sampling variance $\Var_{\gam1}(f)/B$, where $\Var_{\gam1}(f):=\Var(f(G))$ for $G\sim\gam1$.
\end{enumerate}
\end{proposition}

\begin{proof}
\emph{(a).} The grid cdf $\hat G_B$ jumps by $1/B$ at each $g_r=\Phi^{-1}\big(\tfrac{r-1/2}B\big)$. On the open interval $(g_{r},g_{r+1})$, $\hat G_B\equiv r/B$ while $\Phi$ increases continuously from $\tfrac{r-1/2}B$ to $\tfrac{r+1/2}B$. The deviation $|\hat G_B-\Phi|$ is therefore largest at the interval's endpoints, where it equals exactly $\tfrac1{2B}$, and the same value is approached on the two unbounded tails ($\Phi<\tfrac1{2B}$ below $g_1$, $\Phi>1-\tfrac1{2B}$ above $g_B$). Hence $\KS(\hat G_B,\gam1)=\sup|\hat G_B-\Phi|=\tfrac1{2B}$ exactly.

\emph{(b).} With distinct within-batch projections, the multiset $\{z_{ij}\}_{i\le B}$ is, for every batch, exactly the fixed grid $\{g_r\}_{r\le B}$, and only the assignment of grid points to samples varies. Any batch statistic of the form $\tfrac1B\sum_if(z_{ij})$ is a symmetric function of this multiset, hence equals the deterministic value $\tfrac1B\sum_rf(g_r)$ regardless of the data, with zero variance, whereas i.i.d.\ Gaussian codes give variance $\Var_{\gam1}(f)/B$. The mean is $\tfrac1B\sum_rg_r=0$ by the grid's sign symmetry $g_{B+1-r}=-g_r$, and the second moment is $\tfrac1B\sum_rg_r^2=s_B$ by definition.
\end{proof}

\begin{theorem}[Batch QC-FM estimates $\pi^{\mathrm{qc}}_{k,U}$]\label{a:consist}
Fix a frame $U$ and assume i.i.d.\ data with finite second moments and continuous slice cdfs. For each sample $x^{(i)}$, collect its batch codes into $z^{(i)}:=(z_{i1},\dots,z_{ik})^\top$ and its population codes into $Z^{(i)}:=(Z_{i1},\dots,Z_{ik})^\top$, where $Z_{ij}:=\Phi^{-1}(F_j(v_{ij}))$ and $F_j:=F_{u_j}$. Let $\hat\pi_B$ and $\tilde\pi_B$ denote the empirical measures of $\{(Uz^{(i)}+\Pperp\epsilon^{(i)},x^{(i)})\}_{i\le B}$ and $\{(UZ^{(i)}+\Pperp\epsilon^{(i)},x^{(i)})\}_{i\le B}$ respectively, the latter being an i.i.d.\ sample from $\pi^{\mathrm{qc}}_{k,U}$.
\begin{enumerate}[label=\textup{(\alph*)},itemsep=2pt,topsep=2pt]
\item The grid deviations are exactly quantifiable, with $\KS(\hat G_B,\gam1)=\frac1{2B}$, with $c_-/B\le1-s_B\le c_+/B$ for universal constants $0<c_-\le c_+$, and with $W_2^2(\hat G_B,\gam1)=O\big(\frac1{B\log B}\big)$.
\item The batch coupling is consistent, in that $\BL(\hat\pi_B,\tilde\pi_B)\to0$ almost surely, so $\hat\pi_B$ converges weakly to $\pi^{\mathrm{qc}}_{k,U}$ and $W_2(\hat\pi_B,\pi^{\mathrm{qc}}_{k,U})\to0$ almost surely.
\item The finite-$B$ error is dimension-free. Pathwise,
\begin{equation*}
\begin{aligned}
W_2^2(\hat\pi_B,\tilde\pi_B)\;&\le\;\frac1B\sum_i\|z^{(i)}-Z^{(i)}\|^2\;\\ &=\;\sum_{j=1}^kW_2^2\big(\hat G_B,\hat\gamma^{(j)}_B\big), \\ \hat\gamma^{(j)}_B:=&\frac1B\sum_i\delta_{Z_{ij}},
\end{aligned}
\end{equation*}
and in expectation
\begin{equation*}
\begin{aligned}
\E\,W_2^2(\hat\pi_B,\tilde\pi_B)\;&\le\;2k\,W_2^2(\hat G_B,\gam1)+2\sum_{j=1}^k\E\,W_2^2(\hat\gamma^{(j)}_B,\gam1)\;\\ &=\;O\Big(\frac{k\log\log B}{B}\Big)
\end{aligned}
\end{equation*}
\citep{BL19}, with constants independent of the ambient dimension $d$. The deterministic grid term is $O(\frac1{B\log B})$, strictly smaller than the i.i.d.-Gaussian term $\asymp\frac{\log\log B}{B}$, so the grid is closer to $\gam1$ than an i.i.d.\ sample of the same size.
\end{enumerate}
\end{theorem}
\begin{proof}
\emph{(a).} The Kolmogorov--Smirnov equality is Proposition~\ref{a:grid}(a). For $s_B$, write $\psi:=(\Phi^{-1})^2$ and note $\int_0^1\psi(\tau)\,d\tau=\E[G^2]=1$ for $G\sim\gam1$, while $s_B=\frac1B\sum_r\psi\big(\tfrac{r-1/2}B\big)$ is the midpoint-rule approximation of this integral over the $B$ quantile cells $\big(\tfrac{r-1}B,\tfrac rB\big]$. The function $\psi$ is strictly convex on $(0,1)$, since $\psi''=2(1+q^2)/\phi(q)^2>0$ with $q=\Phi^{-1}$ and $\phi$ the standard normal density. A convex function lies above its tangent at the cell midpoint, and the tangent integrates over the cell to exactly the midpoint value times the cell width, so each cell integral is at least the midpoint contribution and $s_B<1$. Summing the per-cell convexity gaps over the interior cells, and comparing the two unbounded edge cells against the integral directly, both contributions are of order $1/B$, which gives
\begin{equation}\label{eq:sBbound}
\frac{c_-}{B}\le 1-s_B\le\frac{c_+}{B}.
\end{equation}
For the $W_2^2$ rate, pair $\hat G_B$ with $\gam1$ by matching quantiles, so that
\[
W_2^2(\hat G_B,\gam1)=\sum_{r=1}^B\int_{(r-1)/B}^{r/B}\big(\Phi^{-1}(\tau)-g_r\big)^2\,d\tau .
\]
On an interior cell the mean value theorem gives $|\Phi^{-1}(\tau)-g_r|\le\tfrac1{2B}\big/\min_{\text{cell}}\phi(\Phi^{-1})$, and the tail asymptotics $\phi(\Phi^{-1}(\tau))\sim\tau\sqrt{2\log(1/\tau)}$ as $\tau\downarrow0$, which follows from the normal tail estimate $1-\Phi(x)\sim \phi(x) / x$ \citep[Eq.~26.2.12]{AS64}, make the cell at rank $r$ contribute $O\big(\tfrac1{Br^2\log(B/r)}\big)$, symmetrically at the upper tail. The sum over cells is dominated by the extreme cells and totals
\begin{equation}\label{eq:W2grid}
W_2^2(\hat G_B,\gam1)=O\Big(\frac1{B\log B}\Big);
\end{equation}
see \citep{BL19} for systematic one-dimensional $W_2$ estimates of this kind. Only the rate is used below.

\emph{(b).} Fix a slice $j$, let $\hat F_{B,j}$ be the empirical cdf of the projections $v_{1j},\dots,v_{Bj}$, and set $\tau_{ij}:=F_j(v_{ij})$ and $\hat\tau_{ij}:=\hat F_{B,j}(v_{ij})-\tfrac1{2B}$, so that $Z_{ij}=\Phi^{-1}(\tau_{ij})$ and $z_{ij}=\Phi^{-1}(\hat\tau_{ij})$. Write $\delta_B:=\|\hat F_{B,j}-F_j\|_\infty+\tfrac1{2B}$, so that $\sup_i|\hat\tau_{ij}-\tau_{ij}|\le\delta_B$ and, by Glivenko--Cantelli, $\delta_B\to0$ almost surely. A test function with $\|f\|_\infty\le1$ and $\Lip(f)\le1$ moves by at most $\min(2,\|z^{(i)}-Z^{(i)}\|)\le\sum_j\min(2,|z_{ij}-Z_{ij}|)$ between corresponding atoms, so
\[
\BL(\hat\pi_B,\tilde\pi_B)\;\le\;\sum_{j=1}^k\frac1B\sum_i\min\big(2,|z_{ij}-Z_{ij}|\big).
\]
Fix $\eta\in(0,\tfrac12)$ and split the samples of slice $j$ by whether $\tau_{ij}\in[\eta,1-\eta]$. On this central part, once $\delta_B<\eta/2$ both $\tau_{ij}$ and $\hat\tau_{ij}$ lie in $[\eta/2,1-\eta/2]$, where $\Phi^{-1}$ is Lipschitz with constant $L_\eta:=1/\phi(\Phi^{-1}(\eta/2))$, since $(\Phi^{-1})'=1/\phi(\Phi^{-1})$ and $\phi(\Phi^{-1})$ attains its minimum over the interval at the endpoints. Hence every central summand is at most $L_\eta\delta_B$, uniformly in $i$, and the central part tends to zero. On the tail part each summand is at most $2$, and the fraction of tail samples tends to $2\eta$ by the strong law, since the $\tau_{ij}$ are i.i.d.\ uniform. Hence $\limsup_B$ of each slice term is at most $4\eta$ for every $\eta$, so $\BL(\hat\pi_B,\tilde\pi_B)\to0$ almost surely. The measure $\tilde\pi_B$ converges weakly to $\pi^{\mathrm{qc}}_{k,U}$ almost surely by Varadarajan's theorem, hence so does $\hat\pi_B$. Weak convergence upgrades to $W_2$ convergence once the second moments converge to that of the limit 
\citep{Villani09}. 
The second moment of $\hat\pi_B$ splits by the orthogonality of $Uz^{(i)}$ and $w^{(i)}:=\Pperp\epsilon^{(i)}$ into $\frac1B\sum_i\|z^{(i)}\|^2+\frac1B\sum_i\|w^{(i)}\|^2+\frac1B\sum_i\|x^{(i)}\|^2$. The first term equals $ks_B$ deterministically, since within each slice the codes form exactly the grid by Proposition~\ref{a:grid}, and $s_B\to1$ by \eqref{eq:sBbound}. The second and third tend to $\E\|w^{(1)}\|^2=\tr\Pperp=d-k$ and $\E\|x^{(1)}\|^2=R^2$ by the strong law. The total $d+R^2$ matches the second moment of $\pi^{\mathrm{qc}}_{k,U}$, which is $\E\|e\|^2+\E\|X\|^2=d+R^2$ by \eqref{eq:enorm}, and the $W_2$ convergence follows.

\emph{(c).} Transporting each atom of $\hat\pi_B$ to the corresponding atom of $\tilde\pi_B$ costs $\|U(z^{(i)}-Z^{(i)})\|^2=\|z^{(i)}-Z^{(i)}\|^2$, which gives the pathwise inequality. Within slice $j$, both $z_{ij}$ and $Z_{ij}$ are strictly increasing in the same within-slice ranks, so the index pairing matches the $r$-th smallest grid point with the $r$-th smallest population code; in one dimension this comonotone pairing is the OT, so $\frac1B\sum_i(z_{ij}-Z_{ij})^2=W_2^2(\hat G_B,\hat\gamma^{(j)}_B)$, which gives the pathwise equality. Triangulating each slice term through $\gam1$, $W_2^2(\hat G_B,\hat\gamma^{(j)}_B)\le2W_2^2(\hat G_B,\gam1)+2W_2^2(\hat\gamma^{(j)}_B,\gam1)$, and taking expectations, the grid term is deterministic and $O(\frac1{B\log B})$ by \eqref{eq:W2grid}, while each $\hat\gamma^{(j)}_B$ is the empirical measure of $B$ i.i.d.\ standard Gaussians, the population codes being exactly $\gam1$ by Theorem~\ref{t:exact}(a), whose expected deviation is $\E\,W_2^2(\hat\gamma^{(j)}_B,\gam1)=O(\log\log B/B)$ \citep{BL19}. Summing the $k$ slices gives the display, and no constant depends on $d$.
\end{proof}

\begin{proposition}[Dependence diagnostic]\label{a:diag}
Let $\hat R_Z:=\frac1{Bs_B}\sum_iz^{(i)}z^{(i)\top}$. With distinct within-batch projections, the following hold.

\begin{enumerate}[label=\textup{(\alph*)},itemsep=2pt,topsep=2pt]

\item $\operatorname{diag}\hat R_Z=I$ exactly. 

\item Under independent sliced coordinates, $\E\|\hat R_Z-I\|_F^2=\frac{k(k-1)}{B-1}$ exactly. 

\item $\hat\kappa_{ij}\to\kappa_{ij}$ almost surely. Hence $\widehat{\mathcal B}:=\|\hat R_Z-I\|_F^2-\frac{k(k-1)}{B-1}$ is the bias-corrected dependence statistic to monitor during training, since subtracting the exact null mean of \textup{(b)} centers it at zero under independence, so positive values indicate genuine slice dependence. Under resampled $U$, average the statistic across steps, never the matrices, because $\E_U[\kappa_{ij}]=0$ by sign symmetry.
\end{enumerate}
\end{proposition}

\begin{proof}
\emph{(a).} See Proposition~\ref{a:grid}(ii).

\emph{(b).} Under the null the rank vectors of distinct slices are independent uniform permutations of $\{1,\dots,B\}$, so
\begin{equation}\label{eq:kappahat}
\hat\kappa_{ij}=\frac1{Bs_B}\sum_m g_{\sigma(m)}g_{\sigma'(m)}\overset d=\frac1{Bs_B}\sum_m g_m g_{\tau(m)}
\end{equation}
for a single uniform permutation $\tau$. Its mean is zero, since $\E[g_{\tau(m)}]=\bar g=0$. For its variance, note that a uniform permutation has the moments $\E[g^2_{\tau(m)}]=\frac1B\sum_p g^2_p=s_B$ and, for $m\ne m'$, $\E[g_{\tau(m)}g_{\tau(m')}]=\frac1{B(B-1)}\sum_{p\ne q}g_pg_q=-\frac{s_B}{B-1}$, the negative correlation coming from $\sum_p g_p=0$. Expanding the square of $\sum_m g_mg_{\tau(m)}$ with these moments,
\begin{equation}\label{eq:rankvar}
\begin{aligned}
\Var\Big(\sum_m g_m g_{\tau(m)}\Big)&=s_B\sum_m g_m^2-\frac{s_B}{B-1}\sum_{m\ne m'}g_mg_{m'}\\
&=Bs_B^2+\frac{s_B}{B-1}\cdot Bs_B=\frac{(Bs_B)^2}{B-1},
\end{aligned}
\end{equation}
where the second equality uses $\sum_{m\ne m'}g_mg_{m'}=\big(\sum_mg_m\big)^2-\sum_mg_m^2=-Bs_B$. This is the classical variance of a simple linear rank statistic \citep{Hoeffding51}, specialized to the scores $g_m$. Dividing by $(Bs_B)^2$ gives $\E[\hat\kappa_{ij}^2]=\tfrac1{B-1}$, and summing over the $k(k-1)$ ordered pairs gives $\E\|\hat R_Z-I\|_F^2=\tfrac{k(k-1)}{B-1}$.

\emph{(c).} The bivariate extension of the consistency argument of Theorem~\ref{a:consist}. Sign symmetry holds because $u_j\mapsto-u_j$ preserves the Haar law and flips ranks.
\end{proof}

\section{A Calibrated Heuristic for Selecting $k$}\label{app:kstar}

This appendix records the rule used to select the number of slices. The theorems supply two exact ingredients. Theorem~\ref{t:price}(a) gives the linear transport-cost gain $2\bar C\,k$, and Theorem~\ref{t:price}(b)--(c) give the dependence law $\E_U\|R_Z-I\|_F^2=k(k-1)\varrho$ together with its transport certificate $g$. The two ingredients act on sample quality through different routes. The gain acts on the training side, through the coupling that Theorems~\ref{t:floor} and \ref{app:straight} analyze, while the prior deviation acts at sampling time, through the endpoint bound of Theorem~\ref{t:exact}(c). No inequality combining the two routes into final sample quality is proved here, so the rule below is a one-parameter heuristic built on the two exact ingredients, with the balance between them calibrated on data rather than derived.

The two theoretical ingredients live on a common squared-transport scale. The gain is the exact reduction of the squared transport cost $c(\bar\pi)$ in Theorem~\ref{t:price}(a). The second term is a plug-in proxy: it evaluates the shape of the per-frame covariance-based certificate from Theorem~\ref{t:price}(c) at the mean dependence $k(k-1)\varrho$ from Theorem~\ref{t:price}(b). Because $g(\E F_U)$ is not itself a proved lower bound on either the mean per-frame mismatch or the frame-averaged source mismatch, we use it only as a calibrated diagnostic. A single dimensionless weight trades this proxy against the exact gain, and we take as our selection score
\begin{equation*}
\begin{gathered}
Q(k)=2\bar C\,k-\omega\,g\big(k(k-1)\varrho\big),\\
2\bar C:=\E_u[2C_u]=\tfrac{R^2}d+1-SW_2^2(p_1,\gam d),
\end{gathered}
\end{equation*}
where $\omega>0$ is that weight. The proxy is approximately quadratic in $k$ below the crossover $k\approx1+\varrho^{-1/2}$ and asymptotically linear in $k$ above it, with slope $\omega\sqrt\varrho$, so the calibrated score can have an interior maximum when this slope exceeds the gain's. Its maximizer $k^\ast$ solves the first-order condition $2\bar C=\omega\,\varrho(2k-1)\,g'\big(k(k-1)\varrho\big)$ numerically. This score omits the frame-disagreement term, finite-batch approximation effects, and hybrid-specific disagreement; consequently, $k^\ast$ is a heuristic selection value rather than a theorem about final sample quality.

Both inputs are computable from the data before training, $\bar C$ by Monte Carlo slicing (Theorem~\ref{t:gain}) and $\varrho$ by Monte Carlo over Haar direction pairs using the debiased statistic of Proposition~\ref{a:diag}. We fix $\omega$ by matching the observed optimum $k^\ast_{\Mix}=16$ of the CIFAR-10 sweep, which gives $\omega\approx13.4$, and we apply the resulting rule to all hybrid variants. The remaining rows of Table~\ref{tab:kstarB} are transferred heuristic selections rather than fitted observations, obtained from each dataset's measured statistics $(\hat\varrho,\bar C)$.

\begin{table}[H]
\centering\small
\caption{Heuristic maximizers $k^\ast$ of the calibrated selection score $Q(k)=2\bar C\,k-\omega\,g(k(k-1)\varrho)$, computed from pre-training data statistics: $\hat\varrho$ by Monte Carlo over Haar direction pairs (debiased, Proposition~\ref{a:diag}) and $\bar C$ by Monte Carlo slicing (Theorem~\ref{t:gain}).}
\label{tab:kstarB}
\begin{tabular}{lrrrr}
\toprule
Dataset & $d$ & $\hat\varrho$ & $\bar C$ & $k^\ast$\\
\midrule
CelebA-64   & $12288$ & $0.0821$ & $1.057$ & $15.0$\\
CIFAR-10    & $3072$  & $0.0677$ & $0.964$ & $16.0^{\dagger}$\\
ImageNet-64 & $12288$ & $0.0673$ & $1.057$ & $21.9$\\
FFHQ-64     & $12288$ & $0.0524$ & $0.994$ & $31.7$\\
\bottomrule
\end{tabular}

\smallskip
{\footnotesize $^{\dagger}$Calibration point, fixing $\omega\approx13.4$ by matching the observed CIFAR-10 optimum.\par}
\end{table}

In practice, we restrict $k$ to powers of two, matching the sweep grid, and select the power nearest to $k^\ast$ on the logarithmic scale, $k=2^{\operatorname{round}(\log_2 k^\ast)}$, which gives $k=16$ on CIFAR-10 and CelebA-64 and $k=32$ on FFHQ-64.

\section{Algorithms, Ablations, and Qualitative Results}\label{app:supplementary}

\subsection{QC-FM and Hybrid Training Algorithms}\label{app:algorithms}

\begin{algorithm}[t]
\caption{Quantile Coupling Flow Matching (QC-FM)}
\label{alg:qcfm}
\begin{algorithmic}[1]
\STATE \textbf{Input:} Data batch \(\{x^{(i)}\}_{i=1}^B \subset \mathbb{R}^d\), number of slices \(k\)
\STATE \textbf{Output:} Source batch \(\{e^{(i)}\}_{i=1}^B\) and frame \(U\)
\STATE Sample \(A \in \mathbb{R}^{d \times k}\) with i.i.d. entries \(A_{ab} \sim \mathcal{N}(0,1)\)
\STATE Compute QR decomposition \(A = QR\), and set \(U \gets Q[:,1\!:\!k]\)
\FOR{\(j = 1, \dots, k\)}
    \FOR{\(i = 1, \dots, B\)}
        \STATE \(v_{ij} \gets u_j^\top x^{(i)}\)
    \ENDFOR
    \STATE Compute ranks \(r_{ij}\) of \(\{v_{1j}, \dots, v_{Bj}\}\)
    \FOR{\(i = 1, \dots, B\)}
        \STATE \(\tau_{ij} \gets (r_{ij} - 0.5)/B\)
        \STATE \(z_{ij} \gets \Phi^{-1}(\tau_{ij})\)
    \ENDFOR
\ENDFOR
\FOR{\(i = 1, \dots, B\)}
    \STATE Form \(z^{(i)} = (z_{i1}, \dots, z_{ik})^\top\)
    \STATE Sample \(\epsilon^{(i)} \sim \mathcal{N}(0, I_d)\)
    \STATE \(e^{(i)} \gets U z^{(i)} + (I_d - UU^\top)\epsilon^{(i)}\)
\ENDFOR
\RETURN \(\{e^{(i)}\}_{i=1}^B, U\)
\end{algorithmic}
\end{algorithm}

\subsection{Training with QC-FM-Based Couplings}

Given the coupled batch
$\mathcal{T}=\{(e^{(i)},x^{(i)})\}_{i=1}^{B}$ produced by either hybrid
strategy, we sample $t\sim q$ from the training-time distribution, construct
$x_t=(1-t)e^{(i)}+tx^{(i)}$, and train the velocity field with the same
conditional flow-matching objective
\[
\mathcal{L}_{\mathrm{FM}}(\theta)
=
\E_{(e,x)\sim\mathcal{T},\, t\sim q}
\Big[
\big\| v_\theta(x_t,t) - (x - e) \big\|^2
\Big].
\]
Thus QC-FM changes only the batchwise source--target pairs; the interpolation
path and regression objective remain unchanged. The hybrids inherit the
structured anchor pairs while retaining exact Gaussian source samples on the
non-anchor slots. Proposition~\ref{t:completion} formalizes the resulting
attenuation of prior mismatch, while Mixture and Adjacency differ only in their
joint pairing laws.

\begin{algorithm}[t]
\caption{Flow Matching Training with QC-FM-Based Hybrid Coupling}
\label{alg:hybrid}
\begin{algorithmic}[1]
\STATE \textbf{Input:} Batch size \(B\), anchor ratio \(p\), number of slices \(k\), training-time distribution \(q\); method \(\in \{\textsc{Mixture},\ \textsc{Adjacency}\}\)
\WHILE{not converged}
    \STATE Sample data batch \(\mathcal{B}_x = \{x^{(i)}\}_{i=1}^B \sim p_{\mathrm{data}}\)
    \STATE Split \(\mathcal{B}_x\) into \(\mathcal{B}_{x,\mathrm{anc}}\) and \(\mathcal{B}_{x,\mathrm{rest}}\) with \(|\mathcal{B}_{x,\mathrm{anc}}| = \lfloor pB \rfloor\)
    \STATE \((\mathcal{B}_{e,\mathrm{anc}},\, U) \gets \textrm{QC-FM}(\mathcal{B}_{x,\mathrm{anc}}, k)\)
    \STATE Form anchor pairs \(\mathcal{S} = \{(e_m^{\mathrm{anc}}, x_m^{\mathrm{anc}})\}\)
    \IF{method is \textsc{Mixture}}
        \STATE Sample \(\widetilde{\mathcal{B}}_{e,\mathrm{rest}} \sim \mathcal{N}(0,I_d)\)
        \STATE Pair \(\widetilde{\mathcal{B}}_{e,\mathrm{rest}}\) randomly with \(\mathcal{B}_{x,\mathrm{rest}}\) to form \(\mathcal{A}\)
        \STATE \(\mathcal{T} \gets \mathcal{S} \cup \mathcal{A}\)
    \ELSIF{method is \textsc{Adjacency}}
        \STATE Assign each \(x \in \mathcal{B}_{x,\mathrm{rest}}\) to its nearest anchor target using \(d_U(x, x_m^{\mathrm{anc}})\)
        \STATE Let \(\mathcal{G}_m\) be the resulting target groups and \(n_m = |\mathcal{G}_m|\)
        \STATE Sample \(\widetilde{\mathcal{B}}_{e,\mathrm{rest}} \sim \mathcal{N}(0,I_d)\)
        \STATE Allocate latents to anchors by a parallel auction, respecting the quotas \(n_m\)
        \STATE Let \(\mathcal{H}_m\) be the allocated latent groups with \(|\mathcal{H}_m| = n_m\)
        \STATE Pair each \(\mathcal{H}_m\) with \(\mathcal{G}_m\) within group to form \(\mathcal{N}\)
        \STATE \(\mathcal{T} \gets \mathcal{S} \cup \mathcal{N}\)
    \ENDIF
    \STATE Sample \(t \sim q\)
    \STATE Construct \(x_t = (1-t)x_0 + t x_1\) for all \((x_0, x_1) \in \mathcal{T}\)
    \STATE Compute \(\mathcal{L}_{\mathrm{FM}}(\theta)\) and update \(\theta\)
\ENDWHILE
\end{algorithmic}
\end{algorithm}

\subsection{Fixed PCA Projection Ablation}\label{app:pca-fixed}

The main experiments redraw a Haar-random orthonormal frame \(U\) at every
training step. To study the combined effect of projection alignment and frame
resampling, we additionally trained matched CIFAR-10 variants using a single
PCA-aligned frame throughout training. We computed the top \(16\) components once from the
\(50{,}000\) training images and their horizontal flips (\(100{,}000\) samples
in total), using the same CHW ordering and \([-1,1]\) normalization as
training. These components explain \(71.66\%\) of the empirical variance. All
other settings are unchanged: \(k=16\), \(B=64\), and \(p=0.2\) for Mixture or
\(p=0.8\) for Adjacency.
Table~\ref{tab:pca_fixed} reports this matched comparison.

\begin{table}[t]
\centering
\caption{Random-resampled versus PCA-fixed projection frames on CIFAR-10 at
the common \(1600\)-epoch endpoint (FID-50K, \(\mathrm{NFE}=50\)).}
\label{tab:pca_fixed}
\small\setlength{\tabcolsep}{5pt}
\begin{tabular}{llcc}
\toprule
Method & Projection frame \(U\) & \(p\) & FID \(\downarrow\) \\
\midrule
QC-FM-Mixture & Random, resampled & 0.2 & 2.10 \\
QC-FM-Mixture & PCA, fixed & 0.2 & 2.07 \\
\midrule
QC-FM-Adjacency & Random, resampled & 0.8 & 2.12 \\
QC-FM-Adjacency & PCA, fixed & 0.8 & 2.16 \\
\bottomrule
\end{tabular}
\end{table}

The PCA-fixed configuration slightly improves Mixture from \(2.10\) to
\(2.07\), but slightly degrades Adjacency from \(2.12\) to \(2.16\). Thus,
the combined change in projection alignment and frame resampling is
hybrid-dependent rather than uniformly beneficial.

\subsection{Rank-Orientation Ablation}\label{app:rank-orientation}

The default QC-FM assigns increasing data ranks to increasing Gaussian quantiles,
forming the monotone one-dimensional coupling. As an orientation control, we
reverse this assignment so that anchor rank \(r\in\{1,\ldots,M\}\) receives
the Gaussian grid quantile at rank \(M+1-r\). This anti-monotone variant
preserves the same Gaussian quantile
multiset, random frame distribution, anchor ratio, and computational cost; it
changes only the orientation of each rank pairing. We train the matched
QC-FM-Mixture control on CIFAR-10 with a fresh random \(U\) at every step,
\((p,k)=(0.2,16)\), and \(B=64\). Table~\ref{tab:rank_orientation} compares it
with the Baseline and the monotone variant.

\begin{table}[t]
\centering
\caption{Rank-orientation ablation on CIFAR-10 at the common \(1600\)-epoch
endpoint (FID-50K, \(\mathrm{NFE}=50\)). QC-FM rows use
Mixture with \((p,k)=(0.2,16)\), \(B=64\). 
}
\label{tab:rank_orientation}
\small\setlength{\tabcolsep}{6pt}
\begin{tabular}{llc}
\toprule
Method & Rank assignment & FID \(\downarrow\) \\
\midrule
Baseline & -- & 2.24 \\
QC-FM-Mixture & Anti-monotone & 2.41 \\
QC-FM-Mixture & Monotone & \textbf{2.10} \\
\bottomrule
\end{tabular}

\end{table}

The anti-monotone control degrades FID to \(2.41\), worse than both the
Baseline (\(2.24\)) and the monotone construction (\(2.10\)), despite using
the same Gaussian quantile grid. This supports that the gain comes from
preserving the correct rank orientation rather than from quantile
regularization alone.

\subsection{Batch-Size Ablation}\label{app:batch-size}

Minibatch OT is known to benefit mainly at large batch sizes
\citep{Zha25}. Table~\ref{tab:overhead} shows that QC-FM's measured coupling
latency grows only mildly over the tested range and remains negligible relative
to a full training step: from $B=64$ to $B=2048$, QC-FM-Mixture increases from
$0.42$ to $0.69$~ms and QC-FM-Adjacency from $2.96$ to $5.51$~ms. To place the
OT-CFM comparison on fair footing, we compare QC-FM and OT-CFM across batch sizes;
Table~\ref{tab:batch_ablation} reports the resulting sweep.

\begin{table}[t]
\centering
\caption{Effect of batch size on CIFAR-10 (FID at a common epoch endpoint for
each setting). QC-FM is compared against OT-CFM as the batch size grows, since
minibatch OT is disadvantaged at small $B$. Best results at each batch size are
bold.}
\label{tab:batch_ablation}
\small\setlength{\tabcolsep}{4pt}
\begin{tabular}{lcccc}
\toprule
Method & $B{=}64$ & $B{=}256$ & $B{=}512$ & $B{=}1024$ \\
\midrule
Baseline & 2.24 & 2.21 & 2.22 & 2.26 \\
OT-CFM (Hungarian) & 2.59 & 2.30 & 2.15 & 2.11 \\
QC-FM-Mixture & \textbf{2.10} & \textbf{2.06} & \textbf{2.02} & \textbf{2.04} \\
\bottomrule
\end{tabular}
\end{table}

The Baseline remains nearly unchanged with \(B\), whereas OT-CFM improves from
\(2.59\) to \(2.11\) and narrows its gap to QC-FM-Mixture from \(0.49\) to
\(0.07\). QC-FM-Mixture nevertheless remains best at every tested batch size,
with its lowest observed FID of \(2.02\) at \(B=512\); at \(B=1024\), its
isolated coupling construction is about \(210\times\) faster than Hungarian
OT (Table~\ref{tab:overhead}).

\subsection{Qualitative Image Results}\label{app:image-qualitative}

FID summarizes distribution match but not per-sample fidelity or failure modes.
Figures~\ref{fig:qualitative_cifar}--\ref{fig:qualitative_imagenet} provide an
illustrative matched-noise comparison between the Baseline and QC-FM-Mixture at
$\mathrm{NFE}=50$. Corresponding grid positions share the same Gaussian initial
state; ImageNet additionally shares the class label. Each pair of method grids
uses the same ordering.

\begin{figure}[p]
\centering
\setlength{\tabcolsep}{2pt}
\begin{tabular}{cc}
Baseline & QC-FM-Mixture \\
\includegraphics[width=.47\textwidth]{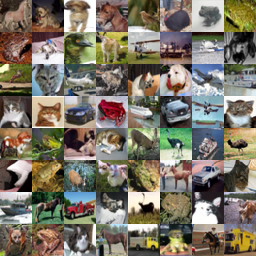} &
\includegraphics[width=.47\textwidth]{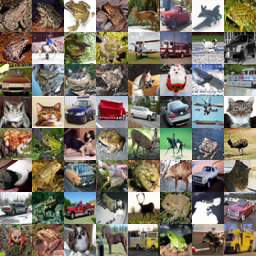}
\end{tabular}
\caption{Matched-noise CIFAR-10 comparisons at the common $1600$-epoch
endpoint ($\mathrm{NFE}=50$). Corresponding cells use identical Gaussian
initial states.}
\label{fig:qualitative_cifar}
\end{figure}

\begin{figure}[p]
\centering
\setlength{\tabcolsep}{2pt}
\begin{tabular}{cc}
Baseline & QC-FM-Mixture \\
\includegraphics[width=.47\textwidth]{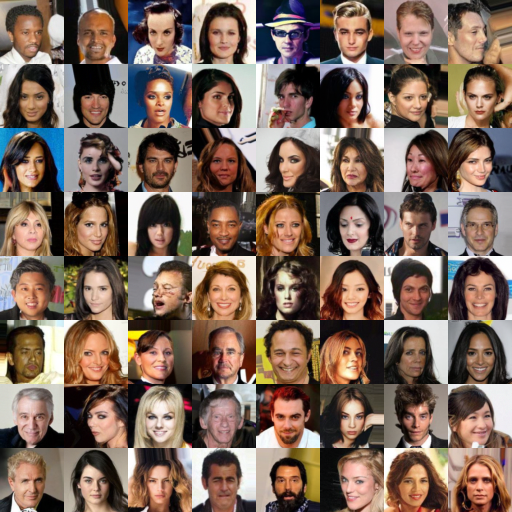} &
\includegraphics[width=.47\textwidth]{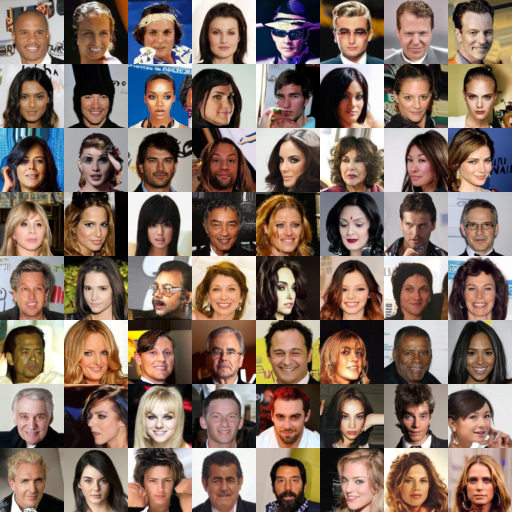}
\end{tabular}
\caption{Matched-noise CelebA-64 comparisons at the common $500$-epoch
endpoint ($\mathrm{NFE}=50$). Corresponding cells use identical Gaussian
initial states.}
\label{fig:qualitative_celeba}
\end{figure}

\begin{figure}[p]
\centering
\setlength{\tabcolsep}{2pt}
\begin{tabular}{cc}
Baseline & QC-FM-Mixture \\
\includegraphics[width=.47\textwidth]{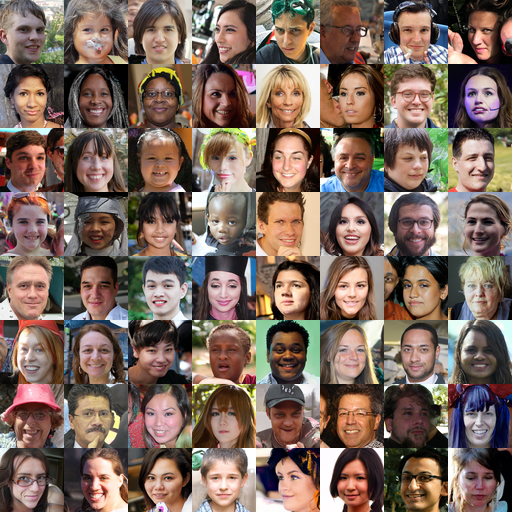} &
\includegraphics[width=.47\textwidth]{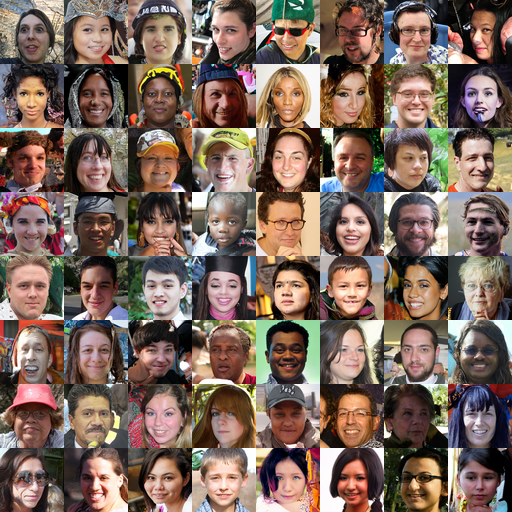}
\end{tabular}
\caption{Matched-noise FFHQ-64 comparisons at the common $1200$-epoch
endpoint ($\mathrm{NFE}=50$). Corresponding cells use identical Gaussian
initial states.}
\label{fig:qualitative_ffhq}
\end{figure}

\begin{figure}[p]
\centering
\setlength{\tabcolsep}{2pt}
\begin{tabular}{cc}
Baseline & QC-FM-Mixture \\
\includegraphics[width=.47\textwidth]{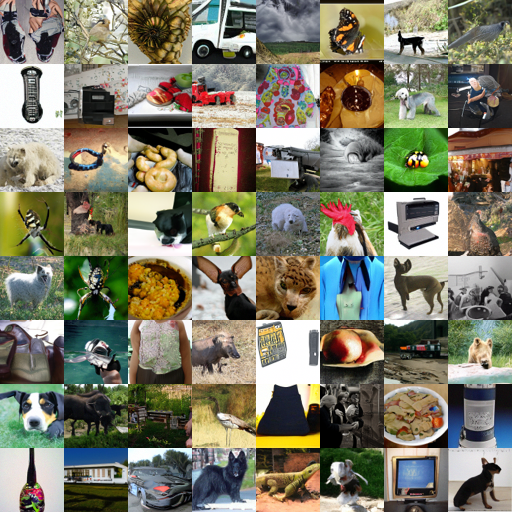} &
\includegraphics[width=.47\textwidth]{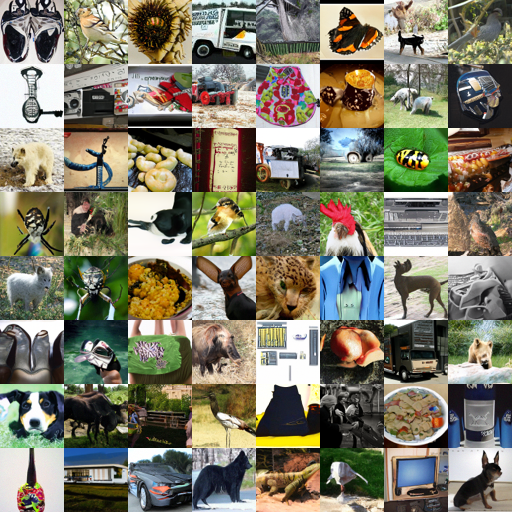}
\end{tabular}
\caption{Matched-input class-conditional ImageNet-64 comparisons at the common
$60$-epoch endpoint under the ADM center-crop protocol. Corresponding cells use
identical Gaussian initial states and class labels ($\mathrm{NFE}=50$,
classifier-free guidance scale $1.5$).}
\label{fig:qualitative_imagenet}
\end{figure}

\end{document}